%% file: main.tex
\documentclass{article}
\usepackage{authblk}

\usepackage[english]{babel}

\usepackage[a4paper,top=2cm,bottom=2cm,left=3cm,right=3cm,marginparwidth=1.75cm]{geometry}
\usepackage{amsmath}
\usepackage{graphicx}
\usepackage[hidelinks]{hyperref}
\usepackage{microtype}
\usepackage{enumitem}
\usepackage{amssymb}
\usepackage{amsthm}
\usepackage{xcolor}
\usepackage{mdframed}
\usepackage{algorithm}
\usepackage{algpseudocode}
\usepackage{booktabs}
\usepackage{subcaption}
\usepackage[nolist,nohyperlinks]{acronym}
\usepackage{tikz}
\usetikzlibrary{arrows.meta,positioning,fit,matrix,calc}
\usepackage{authblk}

\usepackage[authoryear,round]{natbib}
\input{commands}

\title{Finite-Time Guarantees for Constrained Reinforcement Learning under Partial Observability}

\author[1]{Valentina Zaccaria}
\author[2]{Philip Jordan\thanks{Equal contribution.}}
\author[2]{Tingting Ni\protect\footnotemark[1]}
\author[1]{Gian Antonio Susto}
\author[2]{Maryam Kargampour}

\affil[1]{University of Padova}
\affil[2]{SYCAMORE, EPFL}

\date{}

\begin{document}
\maketitle

\begin{abstract}
We study reinforcement learning in constrained partially observable Markov decision processes (CPOMDPs), where an agent maximizes reward subject to a cost constraint using only incomplete information about the underlying state. The lack of state observability may require optimal policies to depend on arbitrarily long observation-action histories, making planning and learning challenging. We leverage recent advances in finite-memory approximations of POMDPs and represent finite history windows as states of a constrained MDP. We then develop a primal--dual algorithm that provides the first non-asymptotic guarantees that jointly bound average reward suboptimality and constraint violation in the original CPOMDP. A key challenge is that this finite-window approximation affects both reward values and the set of feasible policies. Under a filter-stability assumption, we derive bounds relating optimality and feasibility in the finite-state model to those in the original CPOMDP, with approximation errors that decay geometrically with the window length. Experiments in partially observable gridworlds demonstrate the effectiveness of our approach.
\end{abstract}

\section{Introduction}
\acp{POMDP} model sequential decision-making problems in which an agent seeks to maximize its expected cumulative reward without directly observing the state of the environment. Instead, the agent acts on observations that provide only partial information about the underlying state. For example, a robot may not know the exact positions of surrounding objects and must instead rely on noisy camera observations~\citep{lauri2022partially}. Despite the success of \acp{POMDP} in domains such as robotics and healthcare \citep{ding2018adaptive,yu2021reinforcement}, maximizing reward alone does not ensure safe behavior, such as collision avoidance in robotics.\looseness-1

Incorporating safety constraints into \acp{POMDP} leads to \acp{CPOMDP}, where agents maximize expected cumulative reward subject to constraints under partial observability. Despite their practical relevance and successful empirical applications~\citep{hogewind2022safe,carr2023safe,huang2025pigdreamer}, theoretical guarantees for learning in \acp{CPOMDP} remain limited. \looseness-1

Existing theoretical work has focused primarily on planning with explicit knowledge of the model, providing guarantees on feasibility~\citep{undurti2010online}, approximate optimality and feasibility in time polynomial in the model size for any fixed horizon~\citep{khonji2019approximability}, and finite-iteration bounds on reward optimality and constraint satisfaction~\citep{kalagarla2022optimal}. Other work proposes an online Monte Carlo planning method that uses a black-box generative simulator and provides an asymptotic convergence result under additional assumptions~\citep{lee2018monte}. However, these methods do not provide finite-time guarantees for both reward optimality and constraint satisfaction when the model is unknown, and the agent learns solely through interaction with the environment.\looseness-1

Addressing this gap requires overcoming a fundamental difficulty inherited from \acp{POMDP}. Since the agent receives only partial observations of the underlying state, an optimal policy need not be Markovian in the current observation and may depend on the entire history of observations and actions~\citep{kaelbling1998planning}. This history dependence makes both planning and learning challenging. Even with a known model, planning is computationally intractable in general~\citep{papadimitriou1987complexity}; with an unknown model, learning a near-optimal policy requires exponentially many samples in the worst case~\citep{krishnamurthy2016pac}.

To obtain tractable solutions, recent work on \acp{POMDP} approximates the full interaction history using a finite window of recent observations and actions~\citep{kara2022near,kara2023convergence,cayci2024finite,anjarlekar2025scalable,jordan2026model}. In particular, \citet{anjarlekar2025scalable} represent truncated histories as states of a \emph{superstate \ac{MDP}}, and, under a uniform filter-stability assumption, show that the superstate \ac{MDP} approximates the original \ac{POMDP}, with the approximation error decreasing with the history length. They also develop an algorithm with finite-time guarantees for learning a near-optimal policy. Building on the same finite-window approximation, \citet{jordan2026model} study model-based learning in the induced superstate \ac{MDP} for the tabular setting and establish improved sample complexity bounds. These results provide a foundation for learning with finite histories, but address only unconstrained reward maximization.

We extend finite-history methods to \acp{CPOMDP} by representing finite history windows as states of a superstate \ac{CMDP}. Our goal is to learn policies in this model while guaranteeing reward optimality and constraint satisfaction in the original \ac{CPOMDP}. This extension requires addressing three technical challenges. First, history truncation affects both reward and cost values. Thus, a policy that is feasible in the superstate \ac{CMDP} may violate a constraint in the original \ac{CPOMDP}. Second, comparing the optimal values of the two constrained problems requires more than bounding the error in individual policy values, since history truncation changes the set of feasible policies. Third, values in the superstate \ac{CMDP} must be estimated from trajectories generated by the original \ac{CPOMDP}, which introduces a mismatch that the learning analysis must account for.

\paragraph{Contributions.} Our contributions are as follows.
\begin{enumerate}
    \item Under a uniform filter stability condition, we establish bounds relating policy values and constrained optima in the superstate \ac{CMDP} to those in the original \ac{CPOMDP}, quantifying the effects of history truncation on optimality and feasibility (Propositions~\ref{lemma:simulation_lemma} and~\ref{lemma:optgap}).
    \item We develop a primal-dual algorithm (Algorithm~\ref{alg:politex}) and establish finite-time bounds on its average optimality gap and average constraint violation in the original \ac{CPOMDP}, accounting for the estimation error arising when learning from \ac{CPOMDP} trajectories (Theorem~\ref{thm:theorem}). 
    \item We evaluate our algorithm on a partially observable Gridworld and observe convergence across different history lengths, in accordance with our analysis (Section~\ref{sec:simulations}). 
\end{enumerate}

\paragraph{Notation.}
Let $\mathbb{R}$ be the set of real numbers, $\mathbb{N}:=\{1,2,\dots\}$ the set of natural numbers, and $\mathbb{N}_0:=\mathbb{N}\cup\{0\}$. For a finite set $\mathcal{X}$, let $|\mathcal{X}|$ denote its cardinality and~$\Delta(\mathcal{X})$ the probability simplex over $\mathcal{X}$. For $p,q\in\Delta(\mathcal{X})$, the total variation (TV) distance is defined as \(\|p-q\|_{\mathrm{TV}}:=\frac{1}{2}\sum_{x\in\mathcal{X}}|p(x)-q(x)|.\)
The indicator $\mathbf{1}\{x\in\mathcal{X}\}$ equals $1$ if $x\in\mathcal{X}$ and $0$ otherwise.
We denote the concatenation of two sequences $x =(x_1,\dots,x_n)$ and $y=(y_1,\dots,y_m)$ by $x\|y := (x_1,\dots,x_n,y_1,\dots,y_m)$. Finally, for $x\in \mathbb{R}$, we write $[x]_+ := \max(0, x)$.

\section{\acs{CPOMDP}}
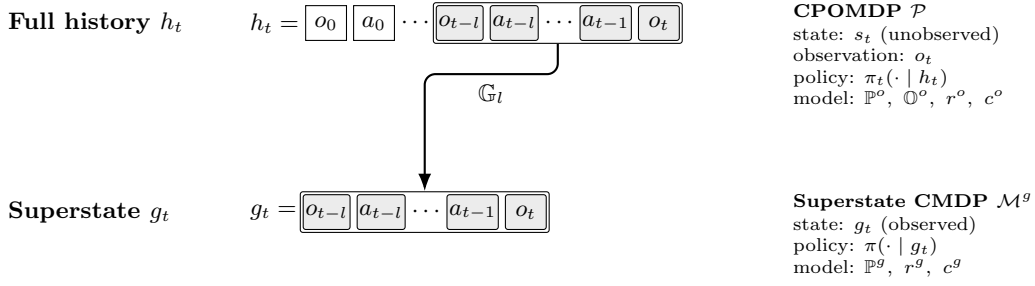
\begin{figure*}[t]
\centering
\begin{tikzpicture}[
    font=\small,
    >=Latex,
    seqbox/.style={
        draw,
        rectangle,
        minimum width=5.4mm,
        minimum height=4.6mm,
        inner sep=.8pt
    },
    superbox/.style={
        seqbox,
        rounded corners=1pt,
        fill=gray!15
    },
    plain/.style={
        draw=none,
        inner sep=0pt
    },
    legend/.style={
        font=\scriptsize,
        align=left,
        inner sep=0pt
    }
]


\node[
    anchor=west,
    font=\small\bfseries
] (htitle) at (0,0) {Full history $h_t$};

\node[
    anchor=west,
    font=\small\bfseries
] (gtitle) at (0,-2.5) {Superstate $g_t$};


\matrix (hseq) [
    matrix of math nodes,
    anchor=west,
    at={(3.2,0)},
    column sep=.8mm,
    nodes={anchor=center}
] {
    |[plain]|    h_t=      &
    |[seqbox]|   o_0       &
    |[seqbox]|   a_0       &
    |[plain]|    \cdots    &
    |[superbox]| o_{t-l}   &
    |[superbox]| a_{t-l}   &
    |[plain]|    \cdots    &
    |[superbox]| a_{t-1}   &
    |[superbox]| o_t       \\
};

\node[
    draw,
    rounded corners=1pt,
    inner sep=.9pt,
    fit={
        (hseq-1-5)
        (hseq-1-6)
        (hseq-1-7)
        (hseq-1-8)
        (hseq-1-9)
    }
] (suffix) {};


\node[
    legend,
    anchor=north west
] (pinfo) at (10.5,0.25) {
    \textbf{\ac{CPOMDP} $\mathcal P$}\\[.2mm]
    state: $s_t$ (unobserved)\\
    observation: $o_t$\\
    policy: $\pi_t(\cdot\mid h_t)$\\
    model: $\oP,\ \oO,\ \oreward,\ \ocost$
};


\matrix (gseq) [
    matrix of math nodes,
    anchor=west,
    at={(3.2,-2.5)},
    column sep=.8mm,
    nodes={anchor=center}
] {
    |[plain]|    g_t=      &
    |[superbox]| o_{t-l}   &
    |[superbox]| a_{t-l}   &
    |[plain]|    \cdots    &
    |[superbox]| a_{t-1}   &
    |[superbox]| o_t       \\
};

\node[
    draw,
    rounded corners=1pt,
    inner sep=.9pt,
    fit={
        (gseq-1-2)
        (gseq-1-3)
        (gseq-1-4)
        (gseq-1-5)
        (gseq-1-6)
    }
] (gsuffix) {};


\node[
    legend,
    anchor=north west
] (ginfo) at (10.5,-2.25) {
    \textbf{Superstate \ac{CMDP} $\smdp$}\\[.2mm]
    state: $g_t$ (observed)\\
    policy: $\pi(\cdot\mid g_t)$\\
    model: $\sP,\ \sreward,\ \scost$
};


\draw[->, thick, rounded corners=4pt]
    (suffix.south)
    -- ++(0,-4mm)
    -| node[pos=.25, below] {$\mathbb G_l$}
    (gsuffix.north);

\end{tikzpicture}

\caption{
Comparison between the CPOMDP $\mathcal P$ (top) and
the superstate CMDP $\smdp$ (bottom).
The truncation operator $\mathbb G_l$ maps the full
history $h_t$ to the superstate $g_t$ by retaining
the length-$l$ suffix, which is then treated as the
observed state of $\smdp$.
}
\label{fig:history_superstate}
\end{figure*}

\label{sec:preliminaries_pomdp}
We consider an infinite-horizon discounted \ac{CPOMDP}, defined by the tuple $$\mathcal{P}=(\mathcal{S}, \mathcal{A}, \mathcal{O},  \oP, \oO, \oreward, \ocost, \gamma, \od, \orho),$$
where $\mathcal{S}, \mathcal{A}$ and $\mathcal{O}$ denote finite state, action, and observation spaces, respectively. The transition kernel is denoted by $\oP: \mathcal{S}\times \mathcal{A} \to \Delta(\mathcal{S})$, where $\oP(\cdot\mid s,a)$ specifies the distribution of the next state given the current state and action. Similarly, the observation kernel is denoted by $\oO:\mathcal{S} \to \Delta(\mathcal{O})$, where $\oO(\cdot\mid s)$ defines the distribution over observations given the current state. The functions $\oreward, \ocost: \mathcal{S}\times \mathcal{A} \to [0,1]$ are the reward and cost functions, respectively. Moreover, $\gamma\in[0,1)$ is the discount factor, $\od\in[0,1/(1-\gamma)]$ is the constraint threshold on the expected discounted cumulative cost, and $\orho\in\Delta(\mathcal{S})$ is the initial state distribution.

The interaction begins by sampling an initial state $s_0\sim\orho$ and an initial observation $o_0\sim\oO(\cdot\mid s_0)$. At each timestep $t\in\mathbb{N}_0$, the agent selects an action $a_t\in\mathcal{A}$ and receives the reward $\oreward(s_t,a_t)$ and cost $\ocost(s_t,a_t)$. The latent state then transitions to $s_{t+1}\sim\oP(\cdot\mid s_t,a_t)$, and the agent receives a new observation $o_{t+1}\sim\oO(\cdot\mid s_{t+1})$.\looseness-1

\paragraph{Histories and policies.} Because the latent states are not directly observable, the agent selects actions based on the observation-action history. For each $t\in\mathbb{N}_0$, let $\mathcal{H}_t = \mathcal{O} \times (\mathcal{A}\times\mathcal{O})^t$ denote the set of possible histories at timestep $t$. An element of $\mathcal{H}_t$ is written as $h_t := (o_0,a_0,o_1,\ldots,a_{t-1},o_t)$ for $t\geq 1$, and $h_0 := o_0$. We let $\mathcal{H} := \bigcup_{t\geq0}\mathcal{H}_t$.
A~history-dependent policy is a sequence of mappings
\begin{equation*}
    \pi = \{\pi_t\}_{t\ge 0}, \quad \pi_t:\mathcal{H}_t\to \Delta(\mathcal{A}),
\end{equation*}
and we denote the class of all such policies by $\Pi^\mathrm{hist}$.

Given a policy $\pi\in\Pi^{\mathrm{hist}}$, for any $u \in \{ \oreward,\ocost \}$, we define the corresponding value function as the expected discounted sum of the stage-wise $u$,
\begin{equation*}
V^{\pi}_u(\orho):=\mathbb{E}_{\pi,\oP,\oO,\orho}\left[\sum_{t=0}^{\infty}\gamma^t u(s_t,a_t)\right].
\end{equation*}
\paragraph{Objective.} Our goal is to find a policy $\pi\in \Pi^{\mathrm{hist}}$ that maximizes $\oVrpi(\orho)$ subject to a constraint on $\oVcpi(\orho)$,
\begin{align}
    \max_{\pi \in \Pi^{\mathrm{hist}}} & \ \oVrpi(\orho) \quad  \text{s.t. }\, \oVcpi(\orho) \le \od. \label{problem_1}
\end{align}
If Problem~(\ref{problem_1}) is feasible (e.g., under Assumption~\ref{ass:slatercpompd} below), there exists an optimal history-dependent policy (see Appendix~\ref{app:existence}). We denote such an optimal solution to Problem~(\ref{problem_1}) by $\opistar$.

Learning an optimal policy for Problem~(\ref{problem_1}) is challenging, as, even without constraints, optimal policies may depend on the full observation-action history, whose length grows unbounded over time~\cite{kaelbling1998planning}. 

\section{Finite-window Approximation}
\label{sec:superstateall}
Our goal is to construct a finite-state \ac{CMDP} that approximates the original \ac{CPOMDP} $\mathcal{P}$. Since near-optimal policies of a finite-state \ac{CMDP} can be found tractably, we then show that, under a suitable condition on the transition and observation kernels of $\mathcal{P}$, such policies are approximately optimal and approximately feasible in $\mathcal{P}$. To this end, in Subsection~\ref{sec:superstate}, we define the finite-state \ac{CMDP}, whose states are finite windows of recent observations and actions~\citep{kara2022near,anjarlekar2025scalable}. In Subsection~\ref{sec:approxguarantee}, we present a sufficient condition ensuring that policies in this \ac{CMDP} have reward and cost values close to those in $\mathcal{P}$, and we quantify how the resulting approximation error depends on the window length $l$.

\subsection{Superstate \acs{CMDP}}
\label{sec:superstate}

 Fix $l \in \mathbb{N}_0$. For any history $h_t \in \mathcal{H}_t$, its superstate is defined as the suffix of $h_t$ containing the $l+1$ most recent observations and $l$ most recent actions,
\begin{equation*}
g_t := \mathbb{G}_l(h_t) = \begin{cases}
(o_{t-l},a_{t-l}, \dots, a_{t-1}, o_t), & t \ge l, \\
h_t, & t < l,
\end{cases}
\end{equation*}
where $\mathbb{G}_l : \mathcal{H} \to \bigcup_{k \le l} \mathcal{H}_k$ is the truncation operator.

The superstate \ac{CMDP} has state space $\mathcal{G} := \bigcup_{k \le l} \mathcal{H}_k$ and action space $\mathcal{A}$, both of which are finite. To define its transition, reward, and cost functions, we associate with each superstate $g \in \mathcal{G}$ the \emph{finite-history belief} $b^g \in \Delta(\mathcal{S})$, given by $b^g(s) := \Pr(s \mid g, \orho)$. We give its precise definition and computation in Appendix~\ref{app:superstate_belief}. Note that $b^g$ is in general only an approximation of the exact posterior over the current latent state, which depends on the full history preceding the window, and we discuss the resulting approximation error in Subsection~\ref{sec:approxguarantee}.

\paragraph{Transition dynamics.} For $g,g' \in \mathcal{G}$ and $a \in \mathcal{A}$,
\begin{equation}
\sP(g' \mid g,a) := \sum_{o \in \mathcal{O}} \mathbf{1}\big(g' = \mathbb{G}_l(g \,\| \,(a,o))\big) \Pr(o \mid b^g, a),
\label{eq:superstate_transition_kernel}
\end{equation}
where $\Pr(o \mid b^g, a) := \sum_{s, s' \in \mathcal{S}} \oO(o \mid s') \oP(s' \mid s,a) b^g(s)$. The initial superstate distribution is $\srho(g) := \mathbf{1}(g \in \mathcal{H}_0) \sum_{s \in \mathcal{S}} \oO(g \mid s) \orho(s)$, where we identify $\mathcal{H}_0$ with $\mathcal{O}$. 

\paragraph{Reward and cost.} The superstate reward and cost are given by $\sreward(g,a) := \sum_{s \in \mathcal{S}} \oreward(s,a) b^g(s)$ and $\scost(g,a) := \sum_{s \in \mathcal{S}} \ocost(s,a) b^g(s)$, respectively. The resulting superstate \ac{CMDP} is $\smdp := (\mathcal{G},\mathcal{A},\sP,\sreward,\scost,\gamma,\srho,\sd)$, and Figure~\ref{fig:history_superstate} illustrates how it differs from $\mathcal{P}$.

\paragraph{Policies and value functions.} A stationary superstate policy is a mapping $\pi:\mathcal{G} \to \Delta(\mathcal{A})$. Let $\sPi:=\{\pi: \mathcal{G}\to \Delta(\mathcal{A})\}$ denote the set of all stationary Markov superstate policies.
Given $\pi \in \sPi$, for any $u \in \{ \sreward,\scost \}$, we define the value function
\begin{equation*}
    \begin{aligned}
        V^{\pi}_u(\srho) := \mathbb{E}_{\pi, \sP,\srho}\Big[\sum\nolimits_{t=0}^\infty \gamma^t u(g_t,a_t) \Big].
    \end{aligned}
\end{equation*}
The superstate \ac{CMDP} problem is given by
\begin{equation}
    \max_{\pi \in \sPi} \sVrpi(\srho) \quad \text{s.t. } \sVcpi(\srho) \le \sd.
    \label{problem:primal_superstate}
\end{equation}
Let $\spiast$ denote an optimal policy for Problem~\eqref{problem:primal_superstate}. The threshold $\sd \in [0,1/(1-\gamma)]$ is a design parameter. Since truncating the history perturbs cost values, choosing $\sd$ below $\od$ can help ensure that policies feasible in the superstate \ac{CMDP} remain approximately feasible in the original \ac{CPOMDP}. We next quantify this approximation error and derive a condition on $\sd$ that ensures this feasibility guarantee.

\subsection{Approximating a \ac{CPOMDP} with a Superstate \ac{CMDP}}
\label{sec:approxguarantee}
The finite-history approach replaces the original \ac{CPOMDP} $\mathcal{P}$ with a superstate \ac{CMDP} $\smdp$. This section examines how accurately the resulting model captures the original problem under a filter-stability condition. First, we bound the differences in reward and cost values between the two models (Proposition~\ref{lemma:simulation_lemma}). We then use this bound to choose the constraint threshold $\sd$ for $\smdp$ and compare the optimal reward values of $\smdp$ and $\mathcal{P}$ under their respective constraints (Proposition~\ref{lemma:optgap}).\looseness-1

Both results rely on controlling how much information about the current latent state is lost when the history is truncated. For any history $h\in\mathcal{H}$, let $b^h\in\Delta(\mathcal{S})$ denote the posterior over current states given $h$, i.e., $b^h(s):=\Pr(s\mid h)$ for all $s \in\mathcal{S}$.

The accuracy of the superstate approximation depends on how close the belief $b^h$ induced by the full history is to the belief $b^g$ induced by its superstate $g=\mathbb{G}_l(h)$, defined in Section~\ref{sec:superstate}. Without additional assumptions, observations from the distant past may affect beliefs over the current state, so that the approximation error cannot be controlled in general. We therefore adopt a \emph{uniform filter stability} assumption, commonly used in the literature on \acp{POMDP}~\citep{kara2023convergence,anjarlekar2025scalable,jordan2026model}. Under this assumption, as stated below, the influence of the past on the current belief decays as new actions and observations are appended to the history.
\begin{assumption}[Uniform Filter Stability]
    Given a \ac{CPOMDP} $\mathcal{P}$, there exists $\kappa \in (0,1)$ such that, for all $h, h' \in \mathcal{H}$, $a \in \mathcal{A}$ and $o \in \mathcal{O}$,
    \begin{equation*}
       \big\| b^{h\|(a,o)} - b^{h'\|(a,o)} \big\|_\mathrm{TV} \le (1-\kappa) \big\| b^h-b^{h'}\big\|_\mathrm{TV}
    \end{equation*}
    where $b^{h\|(a,o)}$ and $b^{h'\|(a,o)}$ are the beliefs obtained by appending the same action--observation pair $(a,o)$ to $h$ and $h'$, respectively\footnote{Our results can also be generalized to a weaker multi-step version of this contraction property.}.
    \label{ass:filter}
\end{assumption}

Assumption~\ref{ass:filter} can be verified for a given \ac{CPOMDP} via sufficient conditions from the literature. E.g., \citet{mcdonald2020exponential} give a condition in terms of the Dobrushin coefficients of the transition and observation kernels, which holds when state transitions are sufficiently mixing relative to the information carried by observations.

\paragraph{Policy Transfer Gap.}
Our first result is a two-way policy transfer bound between $\smdp$ and $\mathcal{P}$. In one direction, it bounds the value-function gap for any superstate policy $\pi\in\sPi$. In the other, it maps any history-dependent policy $\pi \in \Pi^\mathrm{hist}$ to a stationary superstate policy $\pi$ with approximately matching reward and cost values.
\begin{proposition}[Policy Transfer Gap]
\label{lemma:simulation_lemma}
    Let $\delta(l) := \frac{2(1-\kappa)^l}{(1-\gamma)^2}$ denote the approximation error.  Under Assumption~\ref{ass:filter}, the following holds.
    \begin{enumerate}[label=(\alph*), leftmargin=*]
        \item For any superstate policy $\pi \in \sPi$ and $u \in \{r,c\}$,
        \[
            |\oVupi(\orho) - \sVupi(\srho)| \le \delta(l).
        \]
        \item For any  policy $\pi \in \Pi^\mathrm{hist}$, there exists a superstate policy $\pi^g \in \sPi$ such that, for both $u \in \{r,c\}$,
        \[
            |\oVupi(\orho) - \sVdaggerpiu(\srho)| \le \delta(l).
        \]
    \end{enumerate}
\end{proposition}
We defer the proof to Appendix~\ref{app:simulation_lemma} and give a brief overview. Part~(a) is a Simulation Lemma argument~\citep{kearns2002near}: under Assumption~\ref{ass:filter}, the rewards, costs and next-observation distributions of the two models are close, which yields a recursion for the worst-case value gap over histories. For Part~(b), the same argument bounds the value gap between $\mathcal{P}$ and an auxiliary model in which $\pi$ still acts on full histories, but rewards, costs, and transitions depend only on the current superstate. In this model, the superstate--action occupancy measure of $\pi$ is induced by a stationary superstate policy $\pi^g$, which therefore attains the same values in~$\smdp$.

\paragraph{Constrained Optimal Value Gap.}
We next bound the loss in constrained optimal reward from solving Problem~(\ref{problem:primal_superstate}) instead of Problem~(\ref{problem_1}). Specifically, we compare the optimal value $\oVrpistar(\orho)$ of $\mathcal{P}$, attained by the optimal history-dependent policy $\opistar$ (Section~\ref{sec:preliminaries_pomdp}), with the optimal value $\sVrpiast(\srho)$ of $\smdp$, attained by the optimal superstate policy $\spiast$. 

By Proposition~\ref{lemma:simulation_lemma}(b), the optimal history-dependent policy $\opistar$ can be mapped to a superstate policy with nearly the same reward and cost values. If the mapped policy is feasible in $\smdp$, its reward cannot exceed $\sVrpiast(\srho)$, giving the desired bound. If it is infeasible, its constraint violation is bounded by the approximation error $\delta(l)$. To control how much its reward may then exceed $\sVrpiast(\srho)$, we use a standard sensitivity result for constrained optimization. This requires the strict feasibility condition stated below, which is common in the \ac{CMDP} literature~\citep{ding2025convergence,efroni2020exploration}.
\begin{assumption}[{Slater Condition on $\mathcal{P}$}]
\label{ass:slatercpompd}
    There exist $\tilde{\pi}\in\Pi^\mathrm{hist}$ and $\oxi>0$ such that $\oVcpibar(\orho)\le \od-\oxi$.
\end{assumption}
The assumption ensures that $\mathcal{P}$ has a policy whose cost is at least $\oxi$ below the threshold $\od$. By Proposition~\ref{lemma:simulation_lemma}(b), there exists a policy in $\sPi$ whose cost in $\smdp$ is at most $\od-\oxi+\delta(l)$. Therefore, if $\sd$ exceeds this bound, $\smdp$ also satisfies the Slater condition, as formalized below.\looseness-1
\begin{lemma}[Slater Condition on $\smdp$]
Under Assumptions~\ref{ass:filter} and~\ref{ass:slatercpompd}, if
\(
\sd>\od-\oxi+\delta(l),
\)
then there exists $\bar{\pi}\in\sPi$ such that
\[
\sVcpibar(\srho)\leq\sd-\sxi,
\]
where $\sxi:=\sd-\od+\oxi-\delta(l)>0$.
    \label{lemma:superstateslater}
\end{lemma}
The margin $\sxi>0$ bounds the optimal dual multiplier of $\smdp$, which in turn controls how much the optimal value of $\smdp$ can increase when its constraint is relaxed. This yields the following bound.
\begin{proposition}[Constrained Optimal Value Gap]\label{lemma:optgap} 
    Suppose that Assumptions~\ref{ass:filter} and~\ref{ass:slatercpompd} hold and that $\sd$ satisfies the condition of Lemma~\ref{lemma:superstateslater}. Then,
    \begin{equation}
        \!\!\!\oVrpistar(\orho) - \sVrpiast(\srho) \le \delta(l) + \frac{[\od - \sd + \delta(l)]_+}{(1-\gamma)\sxi}.
        \label{eq:proofii}
    \end{equation}
\end{proposition}
The proof is given in Appendix~\ref{app:optimal_value}. The proposition bounds how much the optimal reward in $\smdp$ can fall below that in $\mathcal{P}$. In particular, if $\sd=\od-\delta(l)$ and the history length $l$ is large enough that $\oxi>2\delta(l)$ to ensure $\sxi>0$, this bound converges to zero as $l$ grows.

\paragraph{Choice of $\sd$.} 
We seek to learn a policy that is near-optimal and (nearly) feasible in $\mathcal{P}$. By Proposition~\ref{lemma:simulation_lemma}(a), choosing $\sd=\od-\delta(l)$ ensures that every policy feasible in $\smdp$ is also feasible in $\mathcal{P}$. For sufficiently large $l$, Lemma~\ref{lemma:superstateslater} ensures that $\smdp$ remains strictly feasible. Our results hold more generally for any $\sd$ satisfying the condition of Lemma~\ref{lemma:superstateslater}, which reveals a trade-off between optimality and feasibility. Increasing $\sd$ reduces the optimality-loss bound in Proposition~\ref{lemma:optgap}, but may allow a policy feasible in $\smdp$ to violate the constraint in $\mathcal{P}$ by up to $[\sd-\od+\delta(l)]_+$, as follows from Proposition~\ref{lemma:simulation_lemma}(a). Theorem~\ref{thm:theorem} makes this trade-off explicit.\looseness-1

 \section{Algorithm}
\label{sec:algorithm}
\subsection{Lagrangian Formulation}
To solve Problem~(\ref{problem:primal_superstate}), we incorporate the cost constraint into the objective using a nonnegative Lagrange multiplier \(\lambda\). This yields the following max–min formulation:\looseness-1
\begin{align}
\max_{\pi\in\sPi}\min_{\lambda\in\mathbb R_{\geq0}}L(\pi,\lambda):=\sVrpi(\srho)-\lambda\bigl(\sVcpi(\srho)-\sd\bigr).\label{eq:superstate_dual}
\end{align}
For a fixed \(\lambda\), the associated dual function is \(D^\lambda(\srho):=\max_{\pi\in\sPi}L(\pi,\lambda)\), and we denote an optimal dual variable by \(\slambdaast\in\arg\min_{\lambda\in\mathbb R_{\geq0}}D^\lambda(\srho)\).

The following lemma establishes strong duality, thereby justifying the Lagrangian formulation, and provides an explicit upper bound on \(\slambdaast\).  Both statements follow from standard \ac{CMDP} arguments~\citep{altman2021constrained,ding2025convergence}, which apply to $\smdp$ since it satisfies the Slater condition by Lemma~\ref{lemma:superstateslater}; for completeness, we include a proof in Appendix~\ref{app:dualityandbound}.

\begin{lemma}\label{lemma_dualityandbound}
Suppose that Assumptions~\ref{ass:filter} and~\ref{ass:slatercpompd} hold and that $\sd$ satisfies the condition of Lemma~\ref{lemma:superstateslater}. Then,
\[
\textnormal{(i)}\, V_{r^g}^{\spiast}(\srho)=D^{\slambdaast}(\srho),
\quad
\textnormal{(ii)}\, \slambdaast\le \lambda^{\text{max}},
\]
where $\lambda^{\text{max}}:=2(\sxi(1-\gamma))^{-1}$.
\end{lemma}
Part~(i) shows that the Lagrangian formulation in~\eqref{eq:superstate_dual} and Problem~\eqref{problem:primal_superstate} have the same optimal value. Part~(ii) shows that the search for the optimal dual variable \(\slambdaast\) can be restricted to a compact interval $\left[0,\lambda^{\text{max}}\right]$. Together, these results motivate a primal–dual method that alternates between improving the policy and performing a projected update of the Lagrange multiplier. We first present these two updates assuming access to the exact value functions of the current policy $\pi^k$ in $\smdp$, namely $\sQlagpik$ and $\sVcpik(\srho)$ (defined below), and then introduce the estimation oracle used to implement them.\looseness-1

\paragraph{Policy update.} 
For a fixed multiplier $\lambda^k$, maximizing the Lagrangian over policies is equivalent to maximizing $V_{r^g-\lambda^k c^g}^{\pi}(\srho)$ under the modified reward $r^g-\lambda^k c^g$. We define the state-action value function of $\pi^k$ for this reward as
\begin{align*}
    &\sQlagpik(g,a) :=\mathbb{E}_{\pi^k,\sP}\!\left[\sum_{t=0}^{\infty}\gamma^t(r^g-\lambda^k c^g)(g_t,a_t)\,\middle|\,g_0=g,\ a_0=a\right].
\end{align*}
We then update the policy using exponential weights:
\begin{align}
\pi^{k+1}(a\mid g) \propto \pi^k(a\mid g)\exp\!\left(\eta_1 \sQlagpik(g,a)\right), \label{eq:politex_update_primal}
\end{align}
where the stepsize $\eta_1>0$ will be specified in Theorem~\ref{thm:theorem}. This is called \textsc{Politex} update~\citep{abbasi2019politex}.\footnote{Under a softmax parameterization, this update can also be interpreted as a natural policy gradient step with an appropriate stepsize.} When the multiplier is fixed, repeatedly applying this update yields an average optimality gap that decreases to zero at a sublinear rate for the corresponding modified reward.

\paragraph{Dual update.} Given the current policy \(\pi^k\), we update the multiplier to penalize violations of the cost constraint. Since the dual problem minimizes the Lagrangian with respect to \(\lambda\), we use the projected gradient descent step
\begin{equation}
    \lambda^{k+1}= \mathrm{Proj}_{[0, \lambda^{\mathrm{max}}]}\big(\lambda^k - \eta_2(\sd-\sVcpik(\srho))\big),
    \label{eq:dual_update}
\end{equation}
where the stepsize $\eta_2>0$ will be specified in Theorem~\ref{thm:theorem}. When the current policy violates the constraint, i.e., \(\sVcpik(\srho)>\sd\), the multiplier increases and imposes a larger cost penalty in the next policy update; otherwise, it decreases.

\subsection{Estimation Oracle}
The policy update in~\eqref{eq:politex_update_primal} requires the state-action values \(\sQlagpik(g,a)\), whereas the dual update in~\eqref{eq:dual_update} requires the cost value \(\sVcpik(\srho)\). Since neither quantity is directly available, we assume access to an oracle that estimates both using trajectories generated by \ac{CPOMDP} \(\mathcal P\). 
\begin{oracle}\label{orcale}
There exists \(\delta_{\mathrm{est}}\geq 0\) such that, at every iteration \(k\), for any policy \(\pi^k\in\sPi\), \(\lambda^k\in[0, \lambda^{\mathrm{max}}]\), and \((g,a)\in\mathcal G\times\mathcal A\), the oracle \(\mathcal B\) returns estimates \(\sQbarlagpik(g,a)\) and \(\sVbarcpik(\srho)\) satisfying
\begin{equation}
\begin{aligned}
    \mathbb E\!\left[
        \left|\sQlagpik(g,a)-\sQbarlagpik(g,a)\right|
    \right]
    &\leq \delta_{\mathrm{est}},\\
    \mathbb E\!\left[
        \left|\sVcpik(\srho)-\sVbarcpik(\srho)\right|
    \right]
    &\leq \delta_{\mathrm{est}},
    \label{ass:oracle}
\end{aligned}
\end{equation}
where the expectations are taken over the randomness of the estimation procedure.
\end{oracle}
The oracle abstraction allows us to accommodate different policy evaluation procedures. It can be instantiated using either the model-based method~\cite{jordan2026model} or the model-free TD policy evaluation method~\cite{anjarlekar2025scalable} (details in Appendix \ref{app:samplecomplexity}). For these two instantiations, the estimation error $\delta_{\mathrm{est}}$ has two sources: estimating values from finite-length trajectories and the mismatch between the full-history dynamics and the finite-history superstate model. Under Assumption~\ref{ass:filter}, the latter error decreases as the window length $l$ increases for both methods.

Note that exact values \(\sQlagpik\) and \(\sVcpik\) lie in the intervals
\[
\left[-\frac{\lambda^{\mathrm{max}}}{1-\gamma},\,\frac{1}{1-\gamma}\right]
\quad\text{and}\quad
\left[0,\,\frac{1}{1-\gamma}\right],
\]
respectively. Before using their estimates in the policy and dual updates, we therefore clip \(\sQbarlagpik\) and \(\sVbarcpik\) to the corresponding intervals. This keeps estimates bounded without increasing estimation errors.

Combining the preceding steps, we obtain Algorithm~\ref{alg:politex} for learning an optimal policy for the original \ac{CPOMDP}. At each iteration, the algorithm updates the policy using \textsc{Politex} and the dual variable using projected gradient descent, based on the value estimates provided by Estimation Oracle~\(\mathcal B\). 
\begin{algorithm}[h]
    \caption{Primal--Dual \textsc{Politex}}
    \begin{algorithmic}[1]
        \State \textbf{Input:} stepsizes $\eta_1, \eta_2 > 0$, number of iterations $K$, constraint threshold $\sd$, projection set \([0, \lambda^{\mathrm{max}}]\), and Estimation Oracle \(\mathcal B\)
        \State \textbf{Initialize:} \(\pi^0(a\mid g)= 1/|\mathcal A|\) for all \((g,a)\in\mathcal G\times\mathcal A\), and \(\lambda^0= 0\)        
        \For{$k = 0, \dots, K-1$}
            \State \parbox[t]{0.9\linewidth}{Get $\sQbarlagpik(g,a), \forall g,a$ and $ \sVbarcpik(\srho)$ from $\mathcal{B}$} 
        
            \For{$(g,a)\in\mathcal G\times\mathcal A$} 
                 \State \parbox[t]{0.9\linewidth}{\(\pi^{k+1}(a | g)\propto\pi^k(a|g) \exp\!\left(\eta_1 \sQbarlagpik(g,a)\right)\)}
            \EndFor
            \State \(\lambda^{k+1} = \mathrm{Proj}_{[0, \lambda^{\mathrm{max}}]} \Bigl( \lambda^k + \eta_2 (\sVbarcpik(\srho) - \sd ) \Bigr)\)
        \EndFor
        \end{algorithmic}
    \label{alg:politex}
\end{algorithm}
\paragraph{Remark.} Since $|\mathcal{G}|=O(|\mathcal{O}|^{l+1}|\mathcal{A}|^{l})$, in each iteration, the tabular policy update requires time and memory exponential in the window length $l$; this reflects the trade-off between approximation accuracy and computational cost. A similar trade-off is encountered in prior work on finite-window methods~\cite{cayci2024finite,anjarlekar2025scalable,jordan2026model}.

\section{Theoretical Guarantees}
\label{sec:guarantees}
In this section, we evaluate the performance of Algorithm~\ref{alg:politex} in the original \ac{CPOMDP} using the average optimality gap and average constraint violation, defined as follows:
\begin{align*}
    \mathrm{OpGap}_K&:=\frac{1}{K}\sum_{k=0}^{K-1}\mathbb{E}\!\left[\oVrpistar(\orho)-\oVrpik(\orho)\right],\\
    \mathrm{CViol}_K&:=\left[\frac{1}{K}\sum_{k=0}^{K-1}\mathbb{E}\!\left[\oVcpik(\orho)-\od\right]\right]_+,
\end{align*}
where the expectations are over the randomness of the estimation oracle, and $\{\pi^k\}_{k=0}^{K-1}$ denotes the policies generated by Algorithm~\ref{alg:politex} over $K$ iterations.
\begin{theorem}
\label{thm:theorem}
Given Estimation Oracle~\(\mathcal B\) and under Assumptions~\ref{ass:filter} and~\ref{ass:slatercpompd}, we run Algorithm~\ref{alg:politex} for \(K\) iterations with constraint threshold
\(
\sd>\od-\oxi+\delta(l)
\) and
stepsizes
\begin{align*}
\eta_1&:=\frac{(1-\gamma)^2\sxi}{(1-\gamma)\sxi+2}\sqrt{\frac{2\log|\mathcal A|}{K}},\\
\eta_2&:=\frac{2}{\sxi(1+(1-\gamma)\sd)\sqrt K}.
\end{align*}
Then,
\begin{small}
 \begin{align*}
&\mathrm{OpGap}_K \leq \mathcal{O}\!\left(\delta(l)+\frac{\delta_{\mathrm{est}}+[\od-\sd+\delta(l)]_+}{\sxi}\right) +\mathcal{O}\!\left(\frac{\sqrt{\log|\mathcal A|}}{\sxi\sqrt K}\right),\\
&\mathrm{CViol}_K \leq \mathcal{O}\!\left([\sd-\od+\delta(l)]_++\delta_{\mathrm{est}}\right)+\mathcal{O}\!\left(\frac{\sqrt{\log|\mathcal A|}}{\sqrt K}\right).
\end{align*}   
\end{small}
\end{theorem}
The complete bounds are given in Appendix~\ref{app:fulltheorem} and full proof in Appendix \ref{app:omittedproofs}.

The bounds combine three sources of error. First, the approximation error $\delta(l)$ decays geometrically with $l$, consistent with finite-history analyses of unconstrained \acp{POMDP}~\citep{anjarlekar2025scalable,jordan2026model}. Second, the bounds scale linearly with the estimation error $\delta_{\mathrm{est}}$ specified in~\eqref{ass:oracle}. Third, the optimization error from running finitely many primal--dual iterations decays as $\mathcal{O}(\sqrt{\log|\mathcal{A}|/K})$, matching rates for primal-dual methods in \acp{CMDP}~\citep{ding2025convergence}.

The threshold $\sd$ trades optimality for feasibility. Lowering $\sd$ reduces $[\sd-\od+\delta(l)]_+$ in the bound on $\mathrm{CViol}_K$, but increases $[\od-\sd+\delta(l)]_+$ and reduces the Slater margin $\sxi$ (Lemma~\ref{lemma:superstateslater}); both effects loosen the bound on $\mathrm{OpGap}_K$.

To our knowledge, this is the first work to provide non-asymptotic learning guarantees that jointly bound average reward suboptimality and constraint violation in \acp{CPOMDP}. Since these guarantees hold on average across iterates, they do not ensure per-iteration feasibility or final-policy performance. Establishing last-iterate guarantees with safe exploration is a key future direction~\cite{pmlr-v258-ni25a,ding2023last}.
\subsection{Proof Sketch}
We split $\mathrm{OpGap}_K$ and $\mathrm{CViol}_K$ into model approximation terms, controlled by the results of Section~\ref{sec:approxguarantee}, and optimization terms in $\smdp$, controlled by a primal--dual analysis based on~\citep{ding2025convergence}, combined with an exponential weights regret bound. Throughout, we omit the dependence on the initial distributions $\orho, \srho$ for simplicity.


We decompose the average optimality gap as
\begin{align*}
    \mathrm{OpGap}_K  = & \underbrace{\frac{1}{K} \sum_{k=0}^{K-1} \mathbb{E}[\sVrpik- \oVrpik]}_{(i)}+ \underbrace{\oVrpistar-\sVrpiast}_{(ii)}  +  \underbrace{\frac{1}{K} \sum_{k=0}^{K-1}\mathbb{E}[\sVrpiast- \sVrpik]}_{(iii)},
\end{align*}
and the constraint violation as
\begin{equation*}
    \begin{aligned}
        \mathrm{CViol}_K=&\Bigg[ \sd -\od+\underbrace{\frac{1}{K }\sum_{k=0}^{K-1} \mathbb{E} [\oVcpik -\sVcpik]}_{(iv)} +\underbrace{\frac{1}{K }\sum_{k=0}^{K-1} \mathbb{E} [\sVcpik-\sd]}_{(v)}\Bigg]_+.     \label{eq:costdecomposition}
    \end{aligned}
\end{equation*}
Terms $(i)$ and $(iv)$ are policy evaluation gaps between $\mathcal{P}$ and $\smdp$, each bounded by $\delta(l)$ by Proposition~\ref{lemma:simulation_lemma}(a). Term $(ii)$ is the gap between the optimal reward values of the two problems, bounded by Proposition~\ref{lemma:optgap}.
Terms $(iii)$ and $(v)$ are the average reward gap and the average constraint violation of the iterates in $\smdp$. To control them, we first prove an average Lagrangian performance bound.\looseness-1

\begin{lemma}[Bounded Average Performance]
\label{lemma:boundedaverage}
Let $\spiast$ be an optimal policy for the superstate $\smdp$. If the oracle $\mathcal{B}$ satisfies inequalities in~\eqref{ass:oracle}, after running Algorithm~\ref{alg:politex} for $K$ iterations, it holds
\begin{equation*}
\frac{1}{K}\sum_{k=0}^{K-1}\mathbb{E}[(\sVrpiast  - \lambda^k \sVcpiast) -( \sVrpik -\lambda^k\sVcpik)] \le \delta_{\mathrm{lag}},
\end{equation*}
where $\delta_\mathrm{lag}:= \frac{2\delta_\mathrm{est}}{1-\gamma }+ \frac{1+\lambda^\mathrm{max}}{(1-\gamma)^2} \sqrt{\frac{2 \log |\mathcal{A}|}{K}}$.
\end{lemma}
The proof uses the Performance Difference Lemma~\citep{kakade2002approximately} to express the Lagrangian value gap through $\sQlagpik$. Adding and subtracting estimates $\sQbarlagpik$ decomposes it into two estimation error terms, bounded by~\eqref{ass:oracle}, and a statewise term comparing cumulative estimated Lagrangian Q-values under $\spiast$ and the iterates $\pi^k$. The latter is bounded by the exponential weights regret bound~\citep{cesa2006prediction}, and choosing $\eta_1$ appropriately yields $\delta_{\mathrm{lag}}$. Unlike the telescoping KL-divergence argument of~\citet{ding2025convergence} and the high-probability analysis of \textsc{Politex}~\citep{abbasi2019politex}, this argument applies the regret bound pathwise and bounds the estimation errors directly in expectation. The full proof is in Appendix~\ref{app:boundedaverage}.

\paragraph{Optimality.} Adding and subtracting $\lambda^k(\sVcpiast-\sVcpik)$ gives
\begin{equation*}
    \begin{aligned}
        (iii) = {} & \frac{1}{K}\sum_{k=0}^{K-1} \mathbb{E}[\lambda^k(\sVcpiast- \sVcpik)] 
    +\frac{1}{K}\sum_{k=0}^{K-1}\mathbb{E}[(\sVrpiast  - \lambda^k \sVcpiast) -( \sVrpik -\lambda^k\sVcpik)].
    \end{aligned}
\end{equation*}
By telescoping the projected dual update, using non-expansiveness of the projection, feasibility of $\spiast$, and~\eqref{ass:oracle}, the first sum is at most $\frac{\eta_2}{2} \big(\frac{1}{1-\gamma} +\sd\big)^2  +  \lambda^\mathrm{max} \delta_\mathrm{est}$. The second sum is at most $\delta_\mathrm{lag}$ by Lemma~\ref{lemma:boundedaverage}. Hence\looseness-1
\begin{equation}
    (iii) \le \delta_\mathrm{lag}+  \frac{\eta_2}{2}\bigg(\frac{1}{1-\gamma}+\sd\bigg)^2 + \lambda^\mathrm{max}\delta_\mathrm{est},
    \label{eq:proofiii}
\end{equation}
and substituting $\eta_2$ gives the bound in Theorem~\ref{thm:theorem}. The detailed proof is given in Appendix~\ref{app:optimalitygap}.

\paragraph{Constraint.} By Proposition~\ref{lemma:simulation_lemma}(a), combined with the fact that $[x+y]_+ \le [x]_++[y]_+$, we have
\begin{equation}
    \mathrm{CViol}_K \le [\sd-\od+\delta(l)]_+ + [(v)]_+.
    \label{eq:cvioldec_2}
\end{equation}
For any fixed $\lambda \in [0, \lambda^{\mathrm{max}}]$, telescoping the projected dual update from $\lambda^0=0$ and using~\eqref{ass:oracle} gives
\begin{equation*}
    \begin{aligned}
        &\frac{1}{K}\sum_{k=0}^{K-1}\mathbb{E}[(\lambda-\lambda^k)(\sVcpik-\sd)] \le \frac{\lambda^2}{2\eta_2K} + \frac{\eta_2}{2} \bigg( \frac{1}{1-\gamma}+\sd \bigg)^2 + \lambda^{\mathrm{max}}\delta_{\mathrm{est}},
    \end{aligned}
\end{equation*}
while Lemma~\ref{lemma:boundedaverage} and feasibility of $\spiast$ give $\frac{1}{K}\sum_{k=0}^{K-1}\mathbb{E}[\sVrpiast - \sVrpik + \lambda^k(\sVcpik-\sd)]\le\delta_\mathrm{lag}$. Adding the two cancels the $\lambda^k$ terms and, using $\lambda\le\lambda^{\mathrm{max}}$, yields
\begin{equation*}
    \frac{1}{K}\sum_{k=0}^{K-1}\mathbb{E}[\sVrpiast - \sVrpik] + \lambda\cdot(v) \le \varepsilon_c,
\end{equation*}
where $\varepsilon_c$ collects $\delta_\mathrm{lag}$ and the right-hand side above with $\lambda=\lambda^{\mathrm{max}}$. Maximizing over $\lambda\in[0, \lambda^{\mathrm{max}}]$ turns $\lambda\cdot(v)$ into $\lambda^\mathrm{max} [(v)]_+$. Since $\lambda^\mathrm{max}\ge 2\slambdaast$ by Lemma~\ref{lemma_dualityandbound}, Lemma~\ref{lemma:ding} (adapted from~\citet[Lemma~5]{ding2025convergence}) gives $[(v)]_+ \le (1-\gamma)\sxi\varepsilon_c$. Substituting this and $\eta_2$ into~\eqref{eq:cvioldec_2} yields the stated result. Full proof is in Appendix~\ref{app:constraint_violation}.

\section{Simulations} 
\label{sec:simulations}
To illustrate how the algorithm’s performance depends on $K$ and $l$ in a controlled setting, we evaluate it on a \(7\times 7\) gridworld adapted from~\citet{sutton1998reinforcement} and shown in Figure~\ref{fig:gridworld}.
The latent state is the agent's position in the grid, and the action set is \(\{\texttt{up},\texttt{right},\texttt{down},\texttt{left}\}\). Both the transition and observation dynamics are stochastic. Instead of observing its position directly, the agent receives one of five observations:
\(
\mathcal O=\{\texttt{safe},\texttt{edge},\texttt{near-hazard},\texttt{unsafe},\texttt{goal}\}.
\)
Because multiple grid cells can generate the same observation, the observation-action history does not generally identify the latent state. The agent receives a reward of \(1\) at the absorbing goal state and \(0\) elsewhere, and incurs a cost of \(1\) in unsafe states and \(0\) elsewhere. Accordingly, the agent seeks to reach the goal as quickly as possible while limiting the expected discounted number of visits to unsafe cells, despite observing the environment only partially.\looseness-1
\begin{figure}[t]
    \centering
\includegraphics[width=0.3\linewidth]{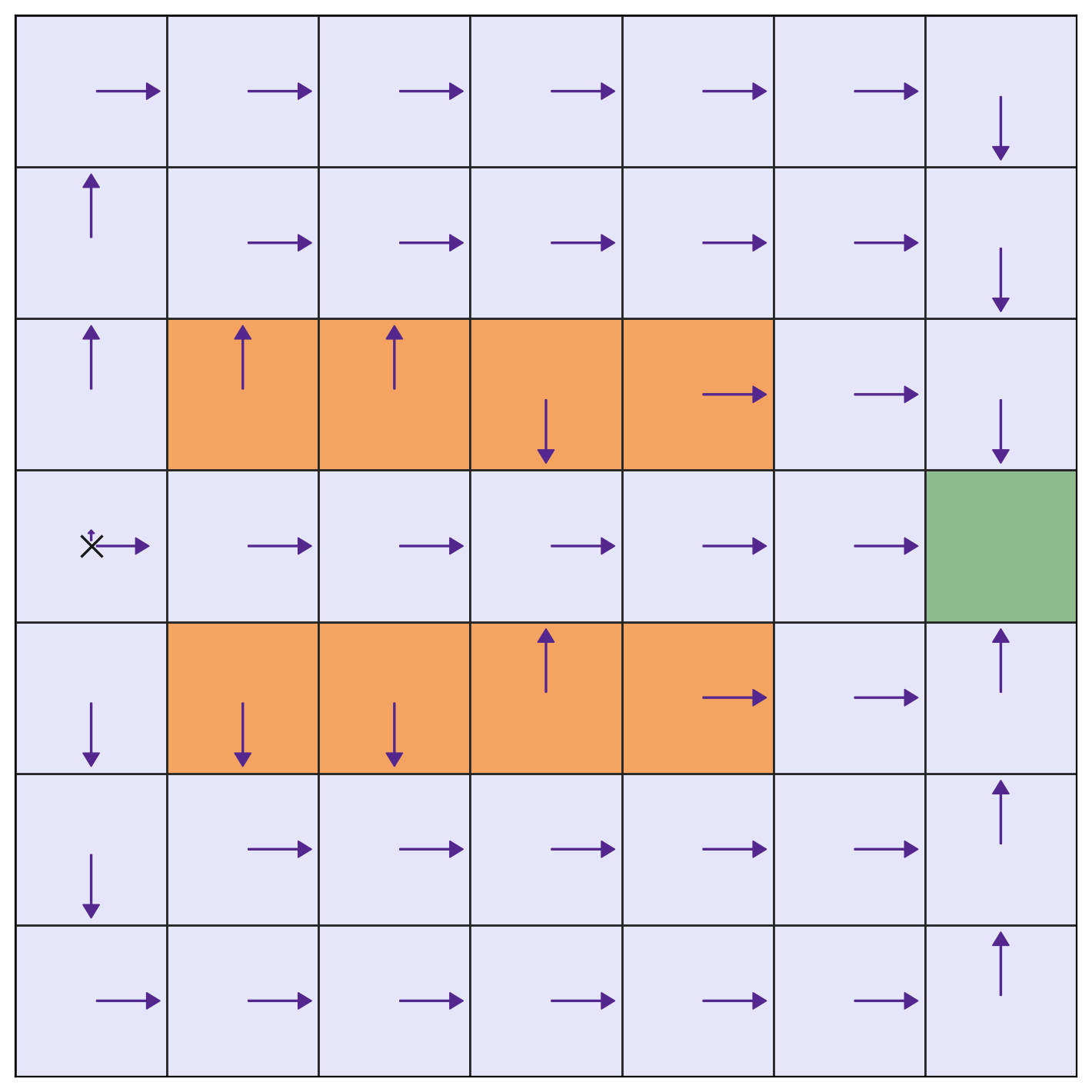}
    \caption{Gridworld: Green cell is the goal, orange cells are unsafe, and \(\times\) is initial position. Arrows are an approximately optimal policy computed using SARSOP \cite{kurniawati2008sarsop}, with lengths proportional to action probabilities.}
    \label{fig:gridworld}
    \vspace{-1em}
\end{figure}

Figure~\ref{fig:grid_results} shows the performance of our algorithm across different history lengths $l$ and iterations $K$. It uses an exact policy-evaluation oracle in the superstate \ac{CMDP} to isolate the effects of $l$ and $K$ from the estimation error.
As \(K\) increases, both the average optimality gap and the average constraint violation approach a neighborhood of zero. Consistent with our theoretical results, the size of this neighborhood is governed by the approximation error $\delta(l)$: increasing \(l\) reduces the approximation error and improves performance.

In Appendix~\ref{app:experiments}, we replace the exact oracle with a sample-based TD(0) oracle, we examine sensitivity to step sizes and the constraint threshold $\sd$, and the performance loss due to partial observability. We provide complete environment and implementation details.

\begin{figure}[t]
    \centering
    \includegraphics[width=0.5\linewidth]
        {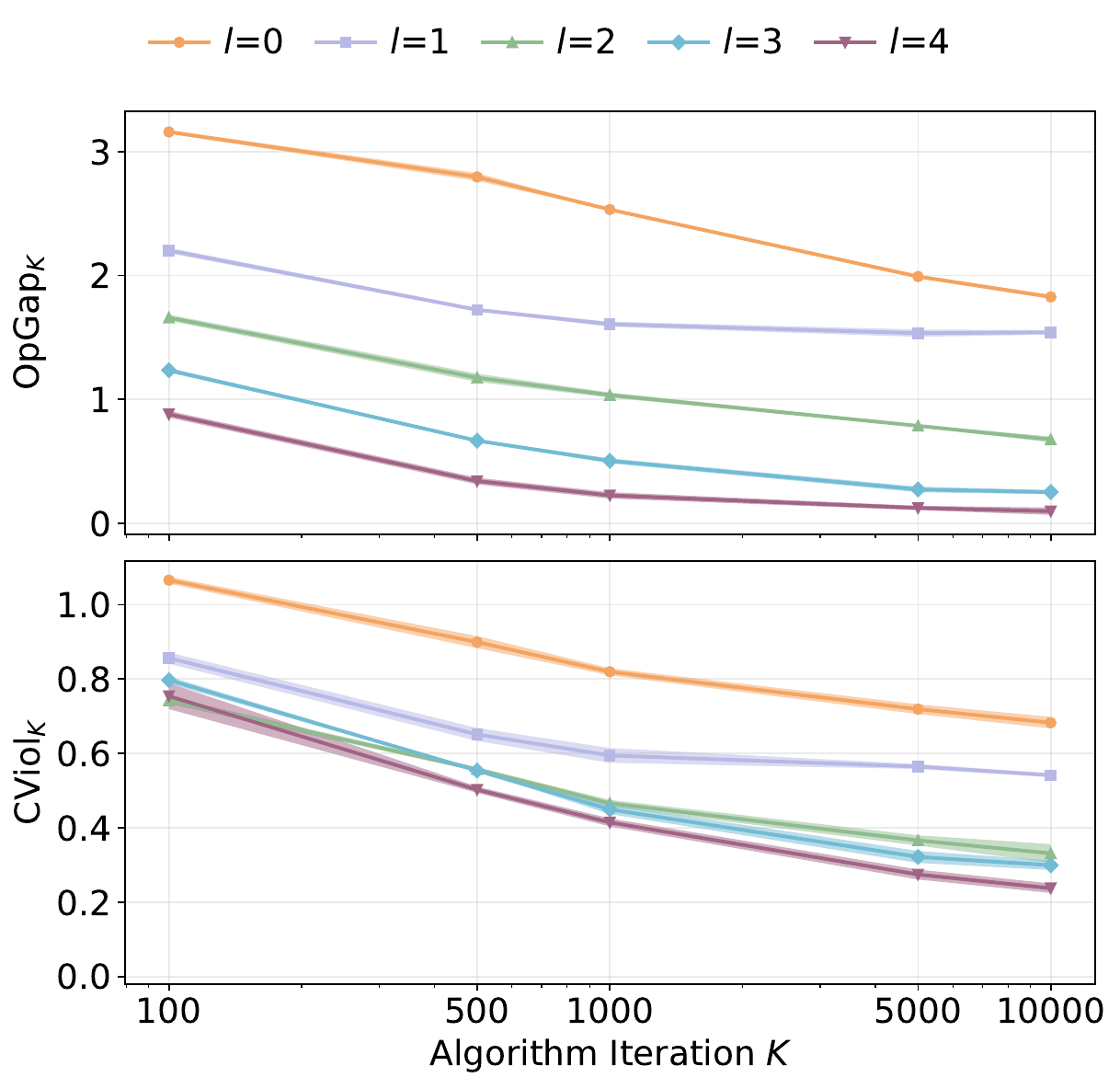}
    \caption{Algorithm performance for different history lengths \(l\). Solid curves show the mean, shaded regions show the evaluation uncertainty. \looseness-1}
    \label{fig:grid_results}
    \vspace{-1em}
\end{figure}

\section{Conclusion \& Future Work}
We addressed RL in \acp{CPOMDP} by representing finite observation-action windows as states of a superstate \ac{CMDP}. We quantified how this approximation affects both reward values and the set of feasible policies, showing that the resulting errors decay geometrically with the window length under filter stability. Building on these bounds, we developed a primal--dual algorithm with non-asymptotic guarantees on average reward suboptimality and constraint violation in the original \ac{CPOMDP}. Experiments on a partially observable constrained Gridworld supported our theoretical results.

Our analysis is limited to tabular superstate \acp{CMDP} and our experiments to Gridworld. Extensions to continuous spaces, realistic benchmarks, and a weaker uniform filter-stability assumption remain for future work. \looseness-1

\bibliography{references}

\clearpage
\appendix
\thispagestyle{empty}

\input{supplementary}

\bibliography{references}

\end{document}

%% file: commands.tex
\newtheorem{lemma}{Lemma}
\newtheorem{theorem}{Theorem}
\newtheorem{assumption}{Assumption}

\newtheorem{proposition}{Proposition}
\newtheorem*{oracle}{Estimation Oracle \(\mathcal B\)}

\newcommand{\oreward}{r^{{o}}}          
\newcommand{\ocost}{c^{{o}}}            
\newcommand{\ou}{u^o}                   
\newcommand{\oP}{\mathbb{P}^{{o}}}      
\newcommand{\oO}{\mathbb{O}^{{o}}}      
\newcommand{\orho}{\rho^{o}}            
\newcommand{\od}{d^{o}}                 
\newcommand{\oxi}{\xi^o}                
\newcommand{\opistar}{\pi^{\star, h}}   

\newcommand{\oVrpi}{V_{r^{o}}^\pi}                  
\newcommand{\oVcpi}{V_{c^{o}}^\pi}                  
\newcommand{\oVupi}{V_{u^o}^\pi}                    
\newcommand{\oVrpistar}{V_{r^{o}}^{\pi^{\star,h}}}  
\newcommand{\oVcpistar}{V_{c^{o}}^{\pi^{\star,h}}}  
\newcommand{\oVcpibar}{V_{c^o}^{\tilde{\pi}}}       

\newcommand{\hreward}{r^{h}}            
\newcommand{\hcost}{c^{h}}              
\newcommand{\hu}{u^h}                   
\newcommand{\hP}{\mathbb{P}^{h}}        
\newcommand{\hrho}{\rho^h}              
\newcommand{\hd}{d^h}                   
\newcommand{\hVrpi}{V_{r^h}^\pi}        
\newcommand{\hVcpi}{V_{c^h}^\pi}        
\newcommand{\hVupi}{V_{u^h}^\pi}        

\newcommand{\smdp}{\mathcal{M}^g}       
\newcommand{\sreward}{r^g}              
\newcommand{\scost}{c^g}                
\newcommand{\su}{u^g}                   
\newcommand{\sP}{\mathbb{P}^g}          
\newcommand{\srho}{\rho^g}              
\newcommand{\sd}{d^g}                   
\newcommand{\sxi}{\xi^g}                

\newcommand{\sPi}{\Pi^g}                
\newcommand{\spiast}{\pi^{*}}           
\newcommand{\slambdaast}{\lambda^*}     
\newcommand{\smuast}{\mu^*}             

\newcommand{\sVrpi}{V_{r^g}^\pi}            
\newcommand{\sVcpi}{V_{c^g}^\pi}            
\newcommand{\sVupi}{V_{u^g}^\pi}            
\newcommand{\sVrpiast}{V_{r^g}^{\pi^{*}}}   
\newcommand{\sVcpiast}{V_{c^g}^{\pi^{*}}}   

\newcommand{\sVrpik}{V_{r^g}^{\pi^k}}       
\newcommand{\sVcpik}{V_{c^g}^{\pi^k}}       
\newcommand{\oVrpik}{V_{r^o}^{\pi^k}}       
\newcommand{\oVcpik}{V_{c^o}^{\pi^k}}       

\newcommand{\sVrpibar}{V_{r^g}^{\bar{\pi}}} 
\newcommand{\sVcpibar}{V_{c^g}^{\bar{\pi}}} 
\newcommand{\sQlagpik}{Q_{r^g-\lambda^kc^g}^{\pi^k}}        
\newcommand{\sVlagkpiast}{V_{r^g-\lambda^k c^g}^{\pi^{*}}}  
\newcommand{\sVlagkpik}{V_{r^g-\lambda^k c^g}^{\pi^k}}      
\newcommand{\sQbarlagpik}{\hat{Q}_{r^g-\lambda^kc^g}^{\pi^k}} 
\newcommand{\sVbarcpik}{\hat{V}_{c^g}^{\pi^k}} 

\newcommand{\sVrpiprime}{V_{r^g}^{\pi'}}                
\newcommand{\sVcpiprime}{V_{c^g}^{\pi'}}                
\newcommand{\sVdaggerpiu}{V_{u^g}^{\pi^g}}              
\newcommand{\sVrpistarg}{V_{r^g}^{(\pi^{\star,h})^g}}   
\newcommand{\sVcpistarg}{V_{c^g}^{(\pi^{\star,h})^g}}   

\acrodef{MDP}[MDP]{Markov Decision Process}
\acrodefplural{MDP}{Markov Decision Processes}
\acrodef{CMDP}[CMDP]{Constrained Markov Decision Process}
\acrodefplural{CMDP}{Constrained Markov Decision Processes}
\acrodef{POMDP}[POMDP]{Partially Observable Markov Decision Process}
\acrodefplural{POMDP}{Partially Observable Markov Decision Processes}
\acrodef{CPOMDP}[CPOMDP]{Constrained Partially Observable Markov Decision Process}
\acrodefplural{CPOMDP}{Constrained Partially Observable Markov Decision Processes}

%% file: supplementary.tex
\begin{center}
    {\Large\bfseries Supplementary Material}
\end{center}

\section{Superstate \ac{CMDP} posterior belief}
\label{app:superstate_belief}
Let $g=(o_0, a_0,o_1, \ldots, a_{k-1}, o_{k})\in \mathcal{H}_k, k \le l$ be a superstate. We associate to it a finite-history belief $b^g \in \Delta(\mathcal{S})$ defined as follows. Starting from the prior on states $\orho$, define
\[
b^g_0(s)=\frac{\orho(s)\oO(o_0| s)}{\sum_{\tilde s\in\mathcal{S}}\orho(\tilde s)\oO(o_0|\tilde s)}.
\]
For $i=0,\ldots,k-1$, recursively define
\[
b^g_{i+1}(s') = \frac{\oO(o_{i+1}|s')\sum_{s \in \mathcal{S}}\oP(s'|s,a_i)b^g_i(s)}{\sum_{\tilde{s}\in \mathcal{S}}\oO(o_{i+1}|\tilde{s}) \sum_{s \in \mathcal{S}}\oP(\tilde{s}|s,a_i) b^g_i(s)}.
\]
We then set $b^g:=b^g_k$.
Equivalently, $b^g$ is obtained by marginalizing over all latent-state trajectories compatible with the observation-action sequence encoded by $g$, starting from the prior $\orho$.

\section{History-state \acs{CMDP}}
\label{app:preliminaries_history}
A classic result~\citep{kaelbling1998planning} is that any \ac{POMDP} can be equivalently formulated as an infinite-state observable \ac{MDP} whose states are full observation-action histories. The same construction easily extends to \acp{CPOMDP} by also accounting for costs. 

Specifically, define the infinite history-state \ac{CMDP} the tuple $$\mathcal{M}^h = (\mathcal{H}, \mathcal{A}, \hP, \hreward, \hcost, \gamma, \hd, \hrho)$$ where $\mathcal{H}$ is the state space, $\mathcal{A}$ is the same action space of the \ac{CPOMDP} $\mathcal{P}$, and $\hd = \od$. The transition kernel $\hP$, the reward and cost functions $\hreward, \hcost$, and the initial history distribution $\hrho$ are defined below. 

The transition kernel $\hP$ specifies the probability of the next history $h'$ given the current history $h$ and action $a$. The next history is obtained by appending $a$ and the next observation $o$ to $h$. Therefore, conditional on $(h,a)$, the only random component of the next history is $o$, and
\[
    \hP (h'|h,a) = \sum_{o \in \mathcal{O}} \mathbf{1}(h'=h\|(a,o)) \Pr(o|h,a)
\]
Since the latent state is not observed, the conditional distribution of the next observation given $h,a$, i.e., $\Pr(o|h,a)$ is obtained by marginalizing over the current and next latent state. Let 
\[
    b^h(s):=\Pr(s_t=s\mid h_t=h),
    \qquad b^h\in\Delta(\mathcal S),
\]
denote the posterior belief induced by the history $h$\footnote{If a history cannot be generated by $\mathcal P$, its
posterior belief may be chosen arbitrarily, since that history is
unreachable under every policy and does not affect values evaluated
from $\hrho$.}. Then, using the structure of the underlying \ac{CPOMDP}
\[
    \Pr(o|h,a) = \sum_{s, s' \in \mathcal{S}} \oO(o|s')\oP(s'|s,a)b^h(s). 
\]

The reward and cost functions are obtained analogously by taking
their conditional expectations over the latent state under $b^h$:
\[
\hreward(h,a)
=
\sum_{s\in\mathcal S}b^h(s)\oreward(s,a),
\qquad
\hcost(h,a)
=
\sum_{s\in\mathcal S}b^h(s)\ocost(s,a).
\]

The initial history consists only of the initial observation. Its initial distribution is therefore: 
\[
\hrho (h) := \mathbf{1}(h \in \mathcal{H}_0) \sum_{s \in \mathcal{S}} \oO(h|s) \orho(s) .
\]

For $\pi \in \Pi^\mathrm{hist}$, define the reward value function:
\[
\hVrpi(h) = \mathbb{E}_{\pi, \hP} \Big[\sum\nolimits_{t=0}^\infty\gamma^t \hreward(h_t, a_t)  \Big|h_0=h \Big],
\]
and its expected reward value with respect to the initial history distribution $\hrho$ as
\[
\hVrpi(\hrho) :=\mathbb{E}_{h \sim \hrho} \big[\hVrpi(h)\big] 
\]
Then, $\hVcpi(h), \hVcpi(\hrho)$ are defined analogously by replacing the reward $\hreward$ with the cost $\hcost$.

By construction,  $\hrho$ reproduces the initial observation
distribution of the original \ac{CPOMDP}, while $\hP$ reproduces the conditional distribution of the next observation given each history-action pair. Hence, under the same history-dependent policy, the two models induce the same distribution over observation-action histories. Moreover, $\hreward$ and $\hcost$ are the conditional
expectations of the original reward and cost. Thus, the two formulations have the same expected reward and cost values when evaluated from their respective initial distributions~\citep{kaelbling1998planning}, namely $\hVrpi(\hrho) = \oVrpi(\orho), \hVcpi(\hrho) = \oVcpi(\orho)$, and they can be used interchangeably. 

\paragraph{Existence of an optimal policy.}
\label{app:existence}
As shown above, Problem~(\ref{problem_1}) is equivalent to the constrained problem in the history-state \ac{CMDP} $\mathcal{M}^h$, which has a countable state space $\mathcal{H}$, a finite action space $\mathcal{A}$, and rewards and costs in $[0,1]$. For such discounted \acp{CMDP}, an optimal policy exists whenever the constrained problem is feasible~\citep{altman2021constrained,hernandez2000constrained}.

\section{Full Theorem Statement}
\label{app:fulltheorem}
Let Assumptions~\ref{ass:filter},~\ref{ass:slatercpompd} hold. Let $\sd$ be such that $\sxi:= \sd-\od+\oxi - \frac{2(1-\kappa)^l}{(1-\gamma)^2}>0$, and set $\lambda^\mathrm{max}= \frac{2}{(1-\gamma)\sxi}$. Run Algorithm~\ref{alg:politex} for $K \in \mathbb{N}$ iterations, using an oracle $\mathcal{B}$ that satisfies inequalities in~\eqref{ass:oracle} and with 
    \begin{align*}
        &\eta_1 = \frac{(1-\gamma)^2\sxi}{(1-\gamma)\sxi+2} \sqrt{\frac{2\log|\mathcal{A}|}{K}},\\
        &\eta_2 = \frac{2}{\sxi(1+(1-\gamma)\sd)\sqrt{K}}
    \end{align*}

Then, the average optimality gap satisfies
\begin{equation*}
     \begin{aligned}
         \frac{1}{K} \sum_{k=0}^{K-1} \mathbb{E}[\oVrpistar(\orho) - \oVrpik(\orho)] \le & { \frac{4(1-\kappa)^l}{(1-\gamma)^2} + \frac{1}{(1-\gamma)\sxi} \left[\od -\sd + \frac{2(1-\kappa)^l}{(1-\gamma)^2}\right]_+ }\\ 
         &
         + {\frac{2(\sxi+1)}{(1-\gamma)\sxi} \delta_\mathrm{est}}\\
         & + {\frac{(1-\gamma)\sxi+2}{(1-\gamma)^3 \sxi} \sqrt{\frac{2\log|\mathcal{A}|}{K}} + \frac{1+(1-\gamma)\sd}{(1-\gamma)^2 \sxi \sqrt{K}}}
     \end{aligned}
\end{equation*}
Moreover, the average constraint violation satisfies
\begin{equation*}
     \begin{aligned}
         \left[\frac{1}{K} \sum_{k=0}^{K-1} \mathbb{E}[\oVcpik(\orho) - \od]\right]_+ \le &{\left[\sd -\od+ \frac{2(1-\kappa)^l}{(1-\gamma)^2}\right]_+} \\
         & + {2(\sxi+1) \delta_\mathrm{est}} \\
         & + {\frac{(1-\gamma)\sxi+2}{(1-\gamma)^2} \sqrt{\frac{2\log|\mathcal{A}|}{K}} + \frac{2(1+(1-\gamma)\sd)}{(1-\gamma)  \sqrt{K}}}.
     \end{aligned}
\end{equation*}

\section{Omitted Proofs}
\label{app:omittedproofs}
\subsection{Proof of Proposition~\ref{lemma:simulation_lemma}}
\label{app:simulation_lemma}

We use the following consequence of Assumption~\ref{ass:filter}, obtained by applying its contraction over the $l$ steps of a superstate.
\begin{lemma}[{\citet[Lemma~1]{anjarlekar2025scalable}}]
    Under Assumption~\ref{ass:filter}, for any history $h \in \mathcal{H}$ and corresponding superstate $g=\mathbb{G}_l(h)$, it holds $\|b^h-b^g\|_\mathrm{TV} \le (1-\kappa)^l$. 
    \label{lemma:beliefbound}
\end{lemma}

For convenience, we prove the proposition in the history-state \ac{CMDP} formulation rather than in the \ac{CPOMDP}, since the former provides an equivalent fully observable formulation, as described in Appendix~\ref{app:preliminaries_history}. We then exploit the equivalence $\hVupi(\hrho)=\oVupi(\orho)$ to transfer the result back to the original formulation.

\textbf{Proof of a)}. 
Fix a stationary superstate policy $\pi \in \sPi$. When deployed in the history-state \ac{CMDP}, it selects actions according to $a_t \sim \pi(\cdot|\mathbb{G}_l(h_t))$. Fix $h \in \mathcal{H}$ and let $g=\mathbb{G}_l(h)$. 
We first show that
\[
    \sup_{h \in \mathcal{H}}  \left|\hVupi (h)-\sVupi (g) \right|\le \delta(l).
\]
where $g=\mathbb{G}_l(h)$ and where $\hVupi(h)$ is the reward/cost value function given a history, defined in Appendix~\ref{app:preliminaries_history}, and $\sVupi(g) := \mathbb{E}_{\pi,\sP}\big[\sum_{t=0}^\infty\gamma^t \su(g_t,a_t)\,\big|\, g_0=g\big]$ is the corresponding superstate value function given an initial superstate $g$. In turn this implies 
\[
 \left|\hVupi (\hrho)-\sVupi(\srho) \right|\le \delta(l).
\]
Finally, because $\hVupi(\hrho)=\oVupi(\orho)$, this will imply statement a).

For $u \in \{r,c\}$ recall that: 
\begin{equation*}
    \begin{alignedat}{2}
        &\hu(h,a) := \sum_{s \in \mathcal{S}}\ou(s,a)b^h(s),
        &&\qquad\su(g,a) := \sum_{s \in \mathcal{S}}\ou(s,a)b^g(s), \\
    \end{alignedat}
\end{equation*}
and 
\begin{equation*}
    \begin{aligned}
        \Pr(o|h,a) &= \sum_{s,s'\in \mathcal{S}} \oO(o|s')\oP(s'|s,a)b^h(s) \\ 
        \Pr(o|b^g,a) &= \sum_{s,s'\in \mathcal{S}} \oO(o|s')\oP(s'|s,a)b^g(s).
    \end{aligned}
\end{equation*}
Since $\ou(s,a) \in [0,1]$, Lemma~\ref{lemma:beliefbound} (which follows from Assumption~\ref{ass:filter}) allows us to bound the reward and cost differences:
\begin{equation}
    |\hu(h,a) - \su (g,a)| = \left|\sum_{s \in \mathcal{S}} \ou(s,a)(b^h(s)-b^g(s))\right| \le \sum_{s \in \mathcal{S}}|b^h(s)-b^g(s)| = 2 \| b^h-b^g\|_\mathrm{TV} \le 2(1-\kappa)^l.
    \label{eq:diffu}
\end{equation}
Similarly
\begin{equation}
    \sum_{o \in \mathcal{O}}|\Pr(o|h,a)-\Pr(o|b^g, a)| = \sum_{o \in \mathcal{O}} \left|\sum_{s, s'\in \mathcal{S}} \oO(o|s')\oP(s'|s,a)(b^h(s)-b^g(s))\right| \le \sum_{s \in \mathcal{S}} |b^h(s)-b^g(s)| \le 2 (1-\kappa)^l.
    \label{eq:diffprob}
\end{equation}

We now consider the Bellman expectation equations in the history-state and superstate \acp{CMDP}, which are, respectively:
\begin{align*}
    \hVupi(h) &= \mathbb{E}_{a \sim \pi(\cdot|g)} \left[\hu (h,a) + \gamma \sum_{o \in \mathcal{O} }\Pr(o|h,a) \hVupi(h\|(a,o))\right] \\
    \sVupi(g) &= \mathbb{E}_{a \sim \pi(\cdot|g)} \left[\su (g,a) + \gamma \sum_{o \in \mathcal{O} }\Pr(o|b^g,a) \sVupi(\mathbb{G}_l(h\|(a,o)))\right] \\
\end{align*}
Subtracting the two equations, adding and subtracting $\Pr(o|h,a) \sVupi(\mathbb{G}_l(h\|(a,o)))$, and applying the triangle inequality gives
\begin{equation}
    \begin{aligned}
        \left|\hVupi(h)- \sVupi(g)\right|  \le  &\, \mathbb{E}_{a \sim \pi(\cdot|g)}|\hu(h,a)-\su(g,a)|+ \gamma \mathbb{E}_{a \sim \pi(\cdot|g)} \sum_{o \in \mathcal{O}} |\Pr(o|h,a)-\Pr(o|b^g,a)| \cdot \left| \sVupi(\mathbb{G}_l(h\|(a,o))) \right|\\
    &   +   \gamma \mathbb{E}_{a \sim \pi(\cdot|g)} \sum_{o \in \mathcal{O}} \Pr(o|h,a)\cdot |\hVupi(h\|(a,o))-\sVupi(\mathbb{G}_l(h\|(a,o)))|   \\ 
    \label{eq:diffbell}
    \end{aligned}
\end{equation}
 Because $\su \in[0,1]$, it holds $|\sVupi(g) |\le1/(1-\gamma)$. Moreover, using~\eqref{eq:diffu}-\eqref{eq:diffprob} in~\eqref{eq:diffbell}, we obtain 
 \begin{align*}
     \left|\hVupi(h)- \sVupi(g)\right| \le 2(1-\kappa)^l + \frac{2\gamma(1-\kappa)^l}{1-\gamma } + \gamma \mathbb{E}_{a \sim \pi(\cdot|g)} \sum_{o \in \mathcal{O}} \Pr(o|h,a)\cdot |\hVupi(h\|(a,o))-\sVupi(\mathbb{G}_l(h\|(a,o)))|
 \end{align*}
Define
\[
    \Delta:= \sup_{h \in \mathcal{H}} \left|\hVupi(h) - \sVupi(\mathbb{G}_l(h))\right|,
\]
which is finite because the value functions are finite. It follows that
\begin{align*}
    \left|\hVupi(h)- \sVupi(g)\right| & \le \frac{2(1-\kappa)^l}{1-\gamma} + \gamma \Delta .
\end{align*}
Taking the supremum over $h \in \mathcal{H}$ gives
\begin{align*}
    \Delta \le \frac{2(1-\kappa)^l}{1-\gamma} + \gamma \Delta \qquad \Rightarrow \qquad \Delta \le \frac{2(1-\kappa)^l}{(1-\gamma)^2} = \delta(l). 
\end{align*} 
Finally, since $\hrho$ and $\srho$ agree
\[
\left|\hVupi(\hrho) - \sVupi(\srho)\right| \le \delta(l).
\]

\textbf{Proof of b).}
Fix an arbitrary history-dependent policy $\pi \in \Pi^\mathrm{hist}$. The argument in part a) cannot be applied directly because the action distribution $\pi(\cdot|h)$ does not generally define a unique policy on the superstate space. In particular, two histories with the same superstate may be associated to different action distributions.

To address this issue, we first construct an auxiliary history-state \ac{CMDP}. Its state is the full history $h \in \mathcal{H}$, so that the original policy $\pi(\cdot|h)$ can be applied without modification. Its reward and cost functions and transition kernel are defined to depend only on the corresponding superstate: 
\begin{align*}
    & u^\mathrm{aux}(h,a) := \su(\mathbb{G}_l(h),a), \\
    &\mathbb{P}^\mathrm{aux}(h'|h,a) := \sum_{o \in \mathcal{O}} \mathbf{1}(h'=h\|(a,o)) \Pr(o|b^{\mathbb{G}_l(h)},a). 
\end{align*}

The same Bellman argument as in part a) can now be applied to $\hVupi(h)$ and $V_{u^\mathrm{aux}}^\pi(h)$, giving
\begin{equation}
    \left| \hVupi(\hrho)-V_{u^\mathrm{aux}}^{\pi}(\hrho)\right|\le \delta(l).
    \label{eq:ciao}
\end{equation}

We next show that, although $\pi$ is history-dependent, its behavior in the auxiliary \ac{CMDP}, in terms of expected reward value and expected cost value, can be reproduced by a stationary policy that only observes the superstate. 
The key observation is that, for expected discounted values, the behavior of a policy is completely characterized by its discounted state-action occupancy measure. Define: 
\[
    \mu^\pi(g,a):=(1-\gamma)\mathbb{E}_{\pi, \mathbb{P}^\mathrm{aux}, \hrho} \left[\sum_{t=0}^\infty \gamma^t \mathbf{1}(\mathbb{G}_l(h_t)=g, a_t=a)\right] \in \Delta(\mathcal{G \times A})
\]
Thus, $\mu^\pi(g,a)$ is the discounted frequency with which the auxiliary process visits superstate $g$ and takes action $a$ under $\pi$. Let: 
\[
    \mu^\pi(g):= \sum_{a \in \mathcal{A}} \mu^\pi(g,a) \in \Delta(\mathcal{G}) 
\]
denote the corresponding state occupancy measure. 

By construction of the auxiliary model, once the current superstate-action pair $(g,a)$ is fixed, the distribution of the next superstate is given by $\sP(\cdot|g,a)$, independently of the particular history that produced $g$. Hence, the state occupancy measure $\mu^\pi(g)$ must satisfy the following Bellman flow constraint~\citep{puterman2014markov}: 
\begin{equation}
    \mu^\pi(g) = (1-\gamma ) \srho(g) + \gamma \sum_{g' \in \mathcal{G}} \sum_{a' \in \mathcal{A}} \mu^\pi(g',a') \sP(g|g',a').
    \label{eq:auxflow}
\end{equation}

We now define a stationary randomized superstate policy $\pi^g \in \sPi$ by 
\[
    \pi^g(a|g) := \frac{\mu^\pi(g,a)}{\mu^\pi(g)},
\]
whenever $\mu^\pi(g) >0$, and arbitrarily otherwise. 
Since by construction $\mu^\pi(g,a) = \mu^\pi(g) \pi^g(a|g)$, we can substitute this identity into the flow constraint in~\eqref{eq:auxflow}: 
\begin{equation*}
    \mu^\pi(g) = (1-\gamma ) \srho(g) + \gamma \sum_{g '\in \mathcal{G}}\sum_{a' \in \mathcal{A}} \mu^\pi(g') \pi^g(a'|g')\sP(g|g',a').
\end{equation*}
But, by definition, also the state occupancy measure induced by $\pi^g$, $\mu^{\pi^g}$ satisfies this constraint:  
\begin{equation*}
    \mu^{\pi^g}(g) = (1-\gamma ) \srho(g) + \gamma \sum_{g' \in \mathcal{G}}\sum_{a' \in \mathcal{A}} \mu^{\pi^g}(g') \pi^g(a'|g')\sP(g|g',a').
\end{equation*}
Because $\gamma <1$, this equation has a unique solution. Hence: 
\[
    \mu^{\pi^g}(g) = \mu^\pi(g)
\]
and so
\[
\mu^{\pi^g}(g,a) = \mu^\pi(g,a).
\]
Finally, since the auxiliary and superstate models have reward and cost as functions of the superstate-action pair $(g,a)$, the expected discounted reward and cost values can be expressed in terms of the corresponding state-action occupancy measure as~\citet{puterman2014markov}:  
\[
    V_{u^\mathrm{aux}}^\pi(\hrho) = \frac{1}{1-\gamma}  \sum_{g \in \mathcal{G}}\sum_{a \in \mathcal{A}} \mu^\pi(g,a) \su(g,a). 
\]
Since the state-action occupancy measures coincide, it follows that
\[
    V_{u^g}^{\pi^g}(\srho) = \frac{1}{1-\gamma}  \sum_{g \in \mathcal{G}}\sum_{a \in \mathcal{A}} \mu^\pi(g,a) \su(g,a)=V_{u^\mathrm{aux}}^\pi(\hrho) 
\]
which concludes the proof.

\subsection{Proof of Lemma~\ref{lemma_dualityandbound}}
\label{app:dualityandbound}
Both parts are standard for \acp{CMDP}~\citep{altman2021constrained,paternain2019constrained},~\citep[Lemma~3]{ding2025convergence}; we include the proof formulated in our superstate setting for completeness. By Lemma~\ref{lemma:superstateslater}, there exists $\bar{\pi}\in\sPi$ with $\sVcpibar(\srho)\le\sd-\sxi$ and $\sxi>0$.

\emph{(i) Strong duality.} For $\pi\in\sPi$, let $\mu^\pi(g,a):=\sum_{t=0}^\infty\gamma^t\Pr_{\pi,\sP}(g_t=g,a_t=a\mid g_0\sim\srho)$ denote its discounted occupancy measure, so that $\sVrpi(\srho)=\langle\mu^\pi,\sreward\rangle$ and $\sVcpi(\srho)=\langle\mu^\pi,\scost\rangle$. Since $\mathcal{G}$ and $\mathcal{A}$ are finite, the set $\mathcal{M}:=\{\mu^\pi:\pi\in\sPi\}$ is a polytope described by the Bellman flow constraints, and every $\mu\in\mathcal{M}$ is induced by the stationary policy $\pi(a\mid g)\propto\mu(g,a)$~\citep{altman2021constrained}. Hence Problem~(\ref{problem:primal_superstate}) is equivalent to the linear program $\max_{\mu\in\mathcal{M}}\{\langle\mu,\sreward\rangle:\langle\mu,\scost\rangle\le\sd\}$, and
\[
D^\lambda(\srho)=\max_{\mu\in\mathcal{M}}\langle\mu,\sreward-\lambda\scost\rangle+\lambda\sd.
\]
This linear program is feasible and bounded, and $\mu^{\bar\pi}$ is strictly feasible. By linear programming duality, the dual problem $\min_{\lambda\ge0}D^\lambda(\srho)$ attains its minimum at some $\slambdaast$, and $V_{r^g}^{\spiast}(\srho)=D^{\slambdaast}(\srho)$.

\emph{(ii) Bound on $\slambdaast$.} By (i) and the definition of $D^{\slambdaast}$,
\[
V_{r^g}^{\spiast}(\srho)=D^{\slambdaast}(\srho)\ge L(\bar\pi,\slambdaast)=\sVrpibar(\srho)+\slambdaast\bigl(\sd-\sVcpibar(\srho)\bigr)\ge\sVrpibar(\srho)+\slambdaast\sxi.
\]
Since $\sreward(g,a)$ is a convex combination of values of $\oreward\in[0,1]$, we have $\sreward\in[0,1]$ and thus $\sVrpi(\srho)\in[0,1/(1-\gamma)]$ for all $\pi\in\sPi$. Rearranging gives
\[
\slambdaast\le\frac{V_{r^g}^{\spiast}(\srho)-\sVrpibar(\srho)}{\sxi}\le\frac{1}{(1-\gamma)\sxi}=\frac{\lambda^{\mathrm{max}}}{2}\le\lambda^{\mathrm{max}}. \qquad\square
\]

\subsection{Proof of Proposition~\ref{lemma:optgap}}
\label{app:optimal_value}
Let $\pi^{\star,h} \in \Pi^\mathrm{hist}$ be an optimal \ac{CPOMDP} policy.
By Proposition~\ref{lemma:simulation_lemma}b), there exists a superstate policy $(\pi^{\star, h})^g$ such that
\[
     \left|\oVrpistar -\sVrpistarg \right| \le   \delta(l), \qquad  \left|\oVcpistar -\sVcpistarg\right| \le   \delta(l).
\]
We have:
\begin{align}
    \oVrpistar - \sVrpiast & = \underbrace{(\oVrpistar - \sVrpistarg)}_{\le \delta(l)} + (\sVrpistarg - \sVrpiast) \le \delta(l) + (\sVrpistarg - \sVrpiast). 
    \label{eq:boh}
\end{align}
We proceed by bounding $(\sVrpistarg - \sVrpiast)$. Since $\pi^{\star, h}$ is feasible in the original \ac{CPOMDP} $\mathcal{P}$, $\oVcpistar \le \od$. 

By Proposition~\ref{lemma:simulation_lemma}b), 
\begin{align*}
    \sVcpistarg \le \oVcpistar +\delta(l) \le \od + \delta(l) .
\end{align*}
We note that $(\pi^{\star, h})^g$ is not necessarily feasible for the superstate $\smdp$. If we define
\[
\Delta_c := [\od + \delta(l) - \sd]_+, 
\]
then $\Delta_c \ge 0$ and $\sVcpistarg \le \sd + \Delta_c$. So, $(\pi^{\star, h})^g$ is feasible for a relaxed version of the superstate problem
\begin{equation}
        \max_{\substack{ \pi \in \Pi^g : \sVcpi \le \sd + \Delta_c}} \ \sVrpi 
    \label{eq:relaxedproblem}
\end{equation}
Now, $(\pi^{\star, h})^g$ is feasible for this relaxed problem, but it is not necessarily optimal. Therefore: 
\begin{align*}
    \sVrpistarg \le \max_{\substack{ \pi \in \Pi^g : \sVcpi \le \sd + \Delta_c}} \ \sVrpi .
\end{align*}
We now use the following constraint-sensitivity result from~\citet[Lemma 14]{tian2024confident}. We restate it under our convention $\sVcpi \le d$ and include the adapted proof for completeness. 
\begin{lemma}
    Assume that the superstate \ac{CMDP} satisfies the Slater condition with margin $\sxi>0$ and let $\slambdaast$ be an optimal multiplier for the problem with threshold $\sd$. For any $\Delta_c \ge 0$, we have: 
    \[
     0 \le \max_{\substack{\pi\in\sPi:\\
                 \sVcpi\le\sd+\Delta_c}} \sVrpi -\sVrpiast\le  \slambdaast \Delta_c \le \frac{1}{(1-\gamma)\sxi} \Delta_c.
    \]
    \label{lemma:lemma14}
\end{lemma}

\emph{Proof}. 
Since $\Delta_c\ge 0$, increasing the threshold enlarges the feasible set, proving the first inequality. \\\\
Now, fix any policy $\pi \in \sPi$ feasible for the relaxed problem. Since $\sVcpi \le \sd+\Delta_c$ and $\slambdaast\ge 0$, we have: 
\begin{align*}
    \sVrpi & = \sVrpi -\slambdaast(\sVcpi-\sd)+\slambdaast(\sVcpi-\sd) \\ 
    & \le \sVrpi -\slambdaast(\sVcpi-\sd) +  \slambdaast \Delta_c \\ 
    & \le \max_{\pi'\in\sPi}\left\{V_{r^g}^{\pi'}-\slambdaast(V_{c^g}^{\pi'}-\sd)\right\}+\slambdaast\Delta_c\\
    & = \sVrpiast + \slambdaast\Delta_c,
\end{align*}
where the last equality follows from strong duality and the optimality of $\slambdaast$. Taking the max over all policies feasible for the relaxed problem gives the result. 
Finally, by the proof of Lemma~\ref{lemma_dualityandbound}(ii) in Appendix~\ref{app:dualityandbound},
\[
\slambdaast \le \frac{1}{(1-\gamma)\sxi}, 
\]
which proves the last inequality $\square$.

Applying Lemma~\ref{lemma:lemma14} gives: 
\begin{equation*}
    \sVrpistarg - \sVrpiast \le \frac{\Delta_c}{(1-\gamma)\sxi}
\end{equation*}

Combining this inequality with~\eqref{eq:boh} yields 
\begin{align*}
    \oVrpistar - \sVrpiast \le \delta(l) + \frac{1}{(1-\gamma)\sxi} \Delta_c = \delta(l) + \frac{1}{(1-\gamma)\sxi} [\od+\delta(l) -\sd]_+
\end{align*}
which proves the lemma.

\subsection{Proof of Lemma~\ref{lemma:boundedaverage}}
\label{app:boundedaverage}
Define 
\[
\sVlagkpiast := \sVrpiast -\lambda^k \sVcpiast, \qquad \sVlagkpik := \sVrpik -\lambda^k \sVcpik.
\]
Let $\smuast \in \Delta(\mathcal{G})$ be the state-normalized discounted occupancy distribution induced by $\spiast$ from  $\srho$: 
\[
    \smuast(g):= (1-\gamma ) \mathbb{E}_{\spiast, \sP, \srho} \left[\sum_{t=0}^{\infty} \gamma ^t \mathbf{1}(g_t=g)\right].
\]
By the Performance Difference Lemma~\citep{kakade2002approximately}: 
\begin{align*}
    \sVlagkpiast - \sVlagkpik  = \frac{1}{1-\gamma } \left\langle \smuast, \mathbb{E}_{a \sim \spiast}[\sQlagpik(\cdot, a)]-\mathbb{E}_{a \sim \pi^k} [\sQlagpik(\cdot,a)]\right\rangle.
\end{align*}
By adding and subtracting the estimated Lagrangian Q-function $\sQbarlagpik$, 
\begin{align*}
    \sVlagkpiast - \sVlagkpik  = &\frac{1}{1-\gamma } \left\langle \smuast, \mathbb{E}_{a \sim \spiast}[\sQlagpik(\cdot, a)-\sQbarlagpik(\cdot,a)]\right\rangle \\
    & + \frac{1}{1-\gamma } \left\langle \smuast, \mathbb{E}_{a \sim \pi^k}[\sQbarlagpik(\cdot,a)-\sQlagpik(\cdot, a)]\right\rangle\\ 
    & + \frac{1}{1-\gamma } \left\langle \smuast, \mathbb{E}_{a \sim \spiast}[\sQbarlagpik(\cdot,a)] -\mathbb{E}_{a \sim \pi^k}[\sQbarlagpik(\cdot, a)]\right\rangle.
\end{align*}
By Estimation Oracle~\(\mathcal B\), after taking the expectation over the estimation-oracle randomness, the first two terms are bounded by $\frac{\delta_\mathrm{est}}{1-\gamma}$ each. Hence: 
\begin{align}
    \sum_{k=0}^{K-1} \mathbb{E}[\sVlagkpiast-\sVlagkpik] &\le \frac{2K\delta_{\mathrm{est}}}{1-\gamma} + \frac{1}{1-\gamma} \mathbb{E}[\langle\smuast, \mathcal{R}_K\rangle],
    \label{eq:regret}
\end{align}
where, for each $g \in \mathcal{G}$: 
\[
\mathcal{R}_K(g) := \sum_{k=0}^{K-1}(\mathbb{E}_{a \sim \spiast}[\sQbarlagpik(g,a)] -\mathbb{E}_{a \sim \pi^k}[\sQbarlagpik(g, a)]).
\]
The remaining task is to bound $\mathcal{R}_K(g)$. 

Fix $g \in \mathcal{G}$ and define: 
\[
    q_k(a):= \sQbarlagpik(g,a).
\]
By construction, the oracle outputs are clipped so that:
\[
    q_k(a)\in \left[-\frac{\lambda^\mathrm{max}}{1-\gamma},\frac{1}{1-\gamma}\right]\subseteq [-R,R], \qquad R:= \frac{1+\lambda^\mathrm{max}}{1-\gamma}.
\]
For fixed $g$, the policy update is an exponential weights update with payoffs $q_k$. So, by the standard exponential weights regret bound~\citep[Theorem 2.2]{cesa2006prediction}, applied to payoffs in an interval of length $2R$ (without rescaling): 
\[
    \max_{a \in \mathcal{A} }\sum_{k=0}^{K-1}  q_k(a) -\sum_{k=0}^{K-1}\mathbb{E}_{a \sim \pi^k}[q_k(a)] \le \frac{\log|\mathcal{A}|}{\eta_1 } + \frac{\eta_1 K R^2}{2}. 
\]
Since 
\[
\sum_{k=0}^{K-1} \mathbb{E}_{a \sim \spiast}[q_k(a)]  \le \max_{a \in \mathcal{A} }\sum_{k=0}^{K-1}  q_k(a),
\]
it follows that 
\[
    \mathcal{R}_K(g)\le \frac{\log|\mathcal{A}|}{\eta_1 } + \frac{\eta_1 K R^2}{2}.
\]
The bound holds for any realization of the sequence $\{q_k\}_{k=0}^{K-1}$. Thus, it also holds after taking expectation over the algorithmic randomness. Moreover, since $\smuast$ is a probability distribution:
\[
    \mathbb{E}[\langle \smuast, \mathcal{R}_K\rangle] \le \frac{\log|\mathcal{A}|}{\eta_1 } + \frac{\eta_1 K R^2}{2}.
\]
Choosing 
\[
    \eta_1 = \frac{1}{R} \sqrt{\frac{2\log|\mathcal{A}|}{K}}
\]
gives 
\[
    \mathbb{E}[\langle \smuast, \mathcal{R}_K\rangle] \le R \sqrt{2K\log|\mathcal{A}|}. 
\]
Substituting this into~\eqref{eq:regret} yields 
\[
    \sum_{k=0}^{K-1} \mathbb{E}[\sVlagkpiast-\sVlagkpik] \le \frac{2K\delta_{\mathrm{est}}}{1-\gamma} + \frac{R}{1-\gamma} \sqrt{2K \log |\mathcal{A}|} = \frac{2K\delta_{\mathrm{est}}}{1-\gamma} + \frac{1+\lambda^\mathrm{max}}{(1-\gamma)^2} \sqrt{2K \log |\mathcal{A}|} 
\]
Finally, dividing both sides by $K$ gives 
\[
\frac{1}{K}\sum_{k=0}^{K-1} \mathbb{E}[\sVlagkpiast-\sVlagkpik] \le \frac{2\delta_{\mathrm{est}}}{1-\gamma} + \frac{1+\lambda^\mathrm{max}}{(1-\gamma)^2} \sqrt{\frac{2\log |\mathcal{A}|}{K}} ,
\]
which concludes the proof.

\subsection{Average Optimality gap}
\label{app:optimalitygap}
We use telescoping, the non-expansiveness of the projection of the dual update and the boundedness of the clipped estimates (Section~\ref{sec:algorithm}), as follows: 
\begin{align*}
    \allowdisplaybreaks
    0 \le \mathbb{E}[(\lambda^K)^2] & = \sum_{k=0}^{K-1} \mathbb{E}[(\lambda^{k+1})^2-(\lambda^k)^2] \\ 
    & = \sum_{k=0}^{K-1} \mathbb{E}[\mathrm{Proj}^2_{[0, \lambda^{\mathrm{max}}]}(\lambda^k + \eta_2(\sVbarcpik - \sd))-(\lambda^k)^2] \\ 
    & \le \sum_{k=0}^{K-1} \mathbb{E}[(\lambda^k + \eta_2(\sVbarcpik - \sd))^2-(\lambda^k)^2] \tag{non-expansiveness of $\mathrm{Proj}_{[0, \lambda^{\mathrm{max}}]}$} \\
    & \le \sum_{k=0}^{K-1} \mathbb{E}[ \eta_2^2  (\sVbarcpik - \sd)^2 + 2\eta_2\lambda^k(\sVbarcpik - \sd)] \\
     & \le K \eta_2^2 \left(\frac{1}{1-\gamma} +\sd\right)^2 + 2\eta_2 \sum_{k=0}^{K-1} \mathbb{E}[\lambda^k(\sVbarcpik - \sd)] \tag{$|\sVbarcpik|\le 1/(1-\gamma)$} \\
     &\le K \eta_2^2 \left(\frac{1}{1-\gamma} +\sd\right)^2 + 2\eta_2 \sum_{k=0}^{K-1} \mathbb{E}[\lambda^k(\sVcpik - \sd)] +2\eta_2 \sum_{k=0}^{K-1} \mathbb{E}[\lambda^k(\sVbarcpik - \sVcpik)] \\ 
     & \le K \eta_2^2 \left(\frac{1}{1-\gamma} +\sd\right)^2 + 2\eta_2 \sum_{k=0}^{K-1} \mathbb{E}[\lambda^k(\sVcpik - \sVcpiast)] +2\eta_2 \sum_{k=0}^{K-1} \mathbb{E}[\lambda^k(\sVbarcpik - \sVcpik)] \tag{$-\sd \le - \sVcpiast$}\\ 
\end{align*}
Arranging terms: 
\begin{align*}
    \sum_{k=0}^{K-1} \mathbb{E}[\lambda^k(\sVcpiast -\sVcpik)] \le  \frac{K\eta_2}{2} \left(\frac{1}{1-\gamma} +\sd\right)^2 + \sum_{k=0}^{K-1} \mathbb{E}[\lambda^k(\sVbarcpik - \sVcpik)].
\end{align*}
Using Estimation Oracle~\(\mathcal B\) and $\lambda^k \le \lambda^\mathrm{max}$, we can conclude: 
\begin{align*}
     \sum_{k=0}^{K-1} \mathbb{E}[\lambda^k(\sVcpiast -\sVcpik)] \le \frac{K\eta_2}{2} \left(\frac{1}{1-\gamma} +\sd\right)^2  + K \lambda^\mathrm{max} \delta_\mathrm{est}.
\end{align*}
Finally, dividing both sides by $K$ gives
\begin{align*}
     \frac{1}{K}\sum_{k=0}^{K-1} \mathbb{E}[\lambda^k(\sVcpiast -\sVcpik)] \le \frac{\eta_2}{2} \left(\frac{1}{1-\gamma} +\sd\right)^2  +  \lambda^\mathrm{max} \delta_\mathrm{est}.
\end{align*}

\subsection{Average constraint violation}
\label{app:constraint_violation}
To bound
\[
\left[\frac{1}{K}\sum_{k=0}^{K-1} \mathbb{E}[\sVcpik -\sd]\right]_+,
\]
we use the dual update to compare the dual iterates with a fixed multiplier $\lambda \in [0, \lambda^{\mathrm{max}}]$. Combined with the Lagrangian average performance bound, in Lemma~\ref{lemma:boundedaverage}, this yields a bound on $\lambda$ times the average constraint residual. Choosing $\lambda=\lambda^\mathrm{max}$ when the residual is positive and $\lambda=0$ otherwise, turns this term into the positive-part average constraint violation. 

For any fixed $\lambda \in [0, \lambda^{\mathrm{max}}]$, by non-expansiveness of the projection and the dual update, we have
\begin{align*}
    \allowdisplaybreaks
    |\lambda^{k+1}-\lambda|^2&  \le \big|\lambda^k-\lambda + \eta_2(\sVbarcpik-\sd)\big|^2 \\ 
    & = |\lambda^k-\lambda|^2 + \eta_2^2 (\sVbarcpik-\sd)^2 + 2\eta_2 (\lambda^k-\lambda)(\sVbarcpik-\sd).
\end{align*}
Rearranging
\begin{align*}
    (\lambda-\lambda^k) (\sVbarcpik-\sd) \le \frac{|\lambda^k-\lambda|^2-|\lambda^{k+1}-\lambda|^2}{2\eta_2} + \frac{\eta_2}{2} (\sVbarcpik-\sd)^2.
\end{align*}
Averaging over $k=0,\ldots,K-1$ and taking expectations gives
\begin{align*}
    \frac{1}{K}\sum_{k=0}^{K-1} \mathbb{E}[(\lambda-\lambda^k) (\sVbarcpik-\sd)] \le \frac{|\lambda^0-\lambda|^2-|\lambda^{K}-\lambda|^2}{2\eta_2K} + \frac{\eta_2}{2K}\sum_{k=0}^{K-1} \mathbb{E}[(\sVbarcpik-\sd)^2]
\end{align*}
Since $\lambda^0=0$, dropping the nonpositive term $-\mathbb{E}|\lambda^K-\lambda|^2$ and using $|\sVbarcpik|\le 1/(1-\gamma)$ from the clipping in Section~\ref{sec:algorithm}, we obtain
\begin{align*}
    \frac{1}{K}\sum_{k=0}^{K-1} \mathbb{E}[(\lambda-\lambda^k) (\sVbarcpik-\sd)] \le \frac{\lambda^2}{2\eta_2 K} + \frac{\eta_2}{2}\left(\frac{1}{1-\gamma} + \sd\right)^2.
\end{align*}
We can now add and subtract the true superstate cost value $\sVcpik$ to the left-hand side, obtaining
\begin{align*}
    \frac{1}{K}\sum_{k=0}^{K-1} \mathbb{E}[(\lambda-\lambda^k) (\sVcpik-\sd)] - \frac{1}{K}\sum_{k=0}^{K-1} \mathbb{E}[(\lambda-\lambda^k) (\sVcpik-\sVbarcpik)] \le \frac{\lambda^2}{2\eta_2 K} + \frac{\eta_2}{2}\left(\frac{1}{1-\gamma} + \sd\right)^2
\end{align*}
By Estimation Oracle~\(\mathcal B\) and because $\lambda, \lambda^k \in [0, \lambda^\mathrm{max}]$, we have
\begin{align}
    \frac{1}{K}\sum_{k=0}^{K-1} \mathbb{E}[(\lambda-\lambda^k) (\sVcpik-\sd)] \le \frac{\lambda^2}{2\eta_2 K} + \frac{\eta_2}{2}\left(\frac{1}{1-\gamma} + \sd\right)^2+ \lambda^\mathrm{max} \delta_\mathrm{est}.
    \label{eq:costproofbo}
\end{align}
Equation~\eqref{eq:costproofbo} bounds the term involving $(\lambda-\lambda^k)(\sVcpik-\sd)$. To obtain a bound involving the fixed multiplier $\lambda$, we are going to bound the other term $\lambda^k(\sVcpik-\sd)$. To do so, we use Lemma~\ref{lemma:boundedaverage} that gives
\begin{align*}
    \frac{1}{K} \sum_{k=0}^{K-1} \mathbb{E}[\sVrpiast-\sVrpik + \lambda^k(\sVcpik-\sVcpiast)] \le \delta_\mathrm{lag}.
\end{align*}
Since $\spiast$ is feasible, i.e., $\sVcpiast\le \sd$, and $\lambda^k \ge 0$, we have $\lambda^k(\sVcpik-\sd)\le \lambda^k(\sVcpik-\sVcpiast)$ and
\begin{align*}
    \frac{1}{K} \sum_{k=0}^{K-1} \mathbb{E}[\sVrpiast-\sVrpik + \lambda^k(\sVcpik-\sd)] \le \delta_\mathrm{lag}.
\end{align*}
Adding this inequality to~\eqref{eq:costproofbo} cancels the terms involving $\lambda^k$, yielding, for any $\lambda \in [0, \lambda^{\mathrm{max}}]$
\begin{align*}
    \frac{1}{K} \sum_{k=0}^{K-1} \mathbb{E}[\sVrpiast-\sVrpik + \lambda(\sVcpik-\sd)] \le \delta_\mathrm{lag} + \frac{\lambda^2}{2\eta_2 K} + \frac{\eta_2}{2}\left(\frac{1}{1-\gamma} + \sd\right)^2+ \lambda^\mathrm{max} \delta_\mathrm{est}.
\end{align*}
Since this inequality holds for any $\lambda \in [0, \lambda^{\mathrm{max}}]$, we choose: 
\begin{align*}
    \lambda = 
    \begin{cases}
        \lambda^\mathrm{max}, & \text{if } \frac{1}{K}\sum_{k=0}^{K-1} \mathbb{E}[\sVcpik-\sd]>0, \\
        0, & \text{otherwise}.
    \end{cases}
\end{align*}
With this choice: 
\begin{align*} \lambda \left( \frac{1}{K}\sum_{k=0}^{K-1} \mathbb{E}[\sVcpik-\sd] \right) = \lambda^{\mathrm{max}} \left[ \frac{1}{K}\sum_{k=0}^{K-1} \mathbb{E}[\sVcpik-\sd] \right]_+,
\end{align*}
and then: 
\begin{align*}
    \frac{1}{K}\sum_{k=0}^{K-1} \mathbb{E}[\sVrpiast-\sVrpik] + \lambda^{\mathrm{max}} \left[ \frac{1}{K}\sum_{k=0}^{K-1} \mathbb{E}[\sVcpik-\sd] \right]_+ \le \underbrace{\delta_{\mathrm{lag}} + \frac{(\lambda^{\mathrm{max}})^2}{2\eta_2K} + \frac{\eta_2}{2} \left( \frac{1}{1-\gamma}+\sd \right)^2 + \lambda^{\mathrm{max}}\delta_{\mathrm{est}}}_{=:\varepsilon_c}.
\end{align*}
By the convexity of the discounted occupancy measure set ~\citep{ding2025convergence}, there exists a stationary randomized policy $\pi'\in\sPi$ such that
\begin{align*}
    \sVrpiprime = \frac{1}{K}\sum_{k=0}^{K-1}\mathbb{E} [\sVrpik], \qquad \sVcpiprime = \frac{1}{K}\sum_{k=0}^{K-1}\mathbb{E} [\sVcpik].
\end{align*}
and 
\begin{align*}
    \sVrpiast- \sVrpiprime + \lambda^\mathrm{max}[\sVcpiprime-\sd]_+ \le \varepsilon_c.
\end{align*}

We now use the upper-bound constraint counterpart of~\citet[Lemma 5]{ding2025convergence}, which we report here.
\begin{lemma}[\textbf{Adapted from ~\citet[Lemma 5]{ding2025convergence}}]
    Assume superstate \ac{CMDP} satisfies Slater condition and let $\slambdaast$ be an optimal dual multiplier. For any $B \ge 2\slambdaast$, if there exists a policy $\pi' \in \sPi$ and $\delta >0$ such that $\sVrpiast-\sVrpiprime + B[\sVcpiprime-\sd]_+ \le \delta$ then $[\sVcpiprime - \sd ]_+ \le \frac{2\delta}{B}$. $\square$
    \label{lemma:ding}
\end{lemma}
By the proof of Lemma~\ref{lemma_dualityandbound}(ii), Appendix~\ref{app:dualityandbound}, \[
\slambdaast \le \frac{1}{(1-\gamma)\sxi},
\]
and choosing $\lambda^\mathrm{max}=2/((1-\gamma)\sxi)$ ensures that $\lambda^\mathrm{max}\ge 2\slambdaast$. 

Applying Lemma~\ref{lemma:ding} with $B=\lambda^\mathrm{max}, \delta=\varepsilon_c$ gives 
\[
[\sVcpiprime-\sd]_+ \le \frac{2\varepsilon_c}{\lambda^\mathrm{max}} = (1-\gamma)\sxi\varepsilon_c.
\]
Using the definition of $\pi'$, we conclude that
\[
\left[\frac{1}{K}\sum_{k=0}^{K-1} \mathbb{E}[\sVcpik-\sd]\right]_+ \le (1-\gamma)\sxi\varepsilon_c.
\]

\section{Conditional Sample Complexity for a Fixed Length}
\label{app:samplecomplexity}

We derive the sample complexity of Algorithm \ref{alg:politex} for a fixed history-window length $l$. Given a target accuracy $\varepsilon>0$, we seek a sufficient number $N_\varepsilon$ of observation-action pairs collected from the original \ac{CPOMDP} $\mathcal P$ such that
\[
\mathrm{OpGap}_K\leq\varepsilon,\qquad \mathrm{CViol}_K\leq\varepsilon.
\]
We consider an implementation of Estimation Oracle $\mathcal B$ based on the projected linear TD procedure of \citet{anjarlekar2025scalable}. We first describe the estimation procedure, then state the finite-time property required from the reward and cost critics, and finally combine it with Theorem \ref{thm:theorem}.

\subsection{Projected Linear TD Learning Estimation Oracle}
\label{app:oracle_options}

Let $\phi:\mathcal G\times\mathcal A\to\mathbb R^d$ be a fixed feature map satisfying
\[
\|\phi(g,a)\|_2\leq1,\qquad \forall(g,a)\in\mathcal G\times\mathcal A.
\]
Fix an ordering of the superstate-action pairs and let $\Phi\in\mathbb R^{|\mathcal G||\mathcal A|\times d}$ be the feature matrix whose row indexed by $(g,a)$ is $\phi(g,a)^\top$. We identify each action-value function with its vector in $\mathbb R^{|\mathcal G||\mathcal A|}$.

For each policy iterate $\pi^k$ and each $u\in\{r,c\}$, the oracle approximates the corresponding action-value function as
\[
\widehat Q_{u^g}^{\pi^k}(g,a):=\phi(g,a)^\top\theta_{u,k},\qquad \theta_{u,k}\in\mathcal B(R):=\{\theta\in\mathbb R^d:\|\theta\|_2\leq R\}.
\]
Algorithm~\ref{alg:tdestimation} gives the complete estimation procedure. The same trajectory is used to update both critics, so collecting reward and cost estimates does not double the number of interactions.

\begin{algorithm}[H]
\caption{Projected Linear TD Estimation Oracle}
\label{alg:tdestimation}
\begin{algorithmic}[1]

\State \textbf{Input:} fixed policy $\pi^k$, feature map $\phi:\mathcal G\times\mathcal A\to\mathbb R^d$, discount factor $\gamma$, projection radius $R$, stepsizes $\{\epsilon_t\}$, initial phase length $\tau'$, and number of TD updates $\tau$

\State Initialize $\theta_r^k=\theta_c^k=0\in\mathbb R^d$

\State Initialize the original CPOMDP from its initial distribution

\State Interact with the original CPOMDP under $\pi^k$ for $\tau'$ steps without updating the critics, obtaining $g_{\tau'}$

\State Sample $a_{\tau'}\sim\pi^k(\cdot\mid g_{\tau'})$

\For{$t=\tau',\ldots,\tau'+\tau-1$}

    \State Execute $a_t$ and observe $r_t$, $c_t$, and $g_{t+1}$

    \State Sample $a_{t+1}\sim\pi^k(\cdot\mid g_{t+1})$

    \For{$u\in\{r,c\}$}

        \State Set $u_t=r_t$ if $u=r$, and $u_t=c_t$ if $u=c$

        \State Compute $\delta_{u,t}^k=u_t+\gamma\phi(g_{t+1},a_{t+1})^\top\theta_u^k-\phi(g_t,a_t)^\top\theta_u^k$

        \State Update $\theta_u^k\gets\operatorname{Proj}_{\mathcal B(R)}\!\left(\theta_u^k+\epsilon_t\delta_{u,t}^k\phi(g_t,a_t)\right)$

    \EndFor

\EndFor

\State Define $\widehat Q_{u^g}^{\pi^k}(g,a)=\phi(g,a)^\top\theta_u^k$ for every $(g,a)\in\mathcal G\times\mathcal A$ and $u\in\{r,c\}$

\State Compute $\widehat Q_{r^g-\lambda^k c^g}^{\pi^k}(g,a)=\widehat Q_{r^g}^{\pi^k}(g,a)-\lambda^k\widehat Q_{c^g}^{\pi^k}(g,a)$

\State Compute $\widehat V_{c^g}^{\pi^k}(\rho^g)=\displaystyle\sum_{g\in\mathcal G}\sum_{a\in\mathcal A}\rho^g(g)\pi^k(a\mid g)\widehat Q_{c^g}^{\pi^k}(g,a)$

\State \textbf{Output:} $\widehat Q_{r^g-\lambda^k c^g}^{\pi^k}$ and $\widehat V_{c^g}^{\pi^k}(\rho^g)$

\end{algorithmic}
\end{algorithm}

Algorithm~\ref{alg:tdestimation} describes a possible implementation of the estimation oracle. Its finite-sample accuracy depends on the feature representation and on the coverage and mixing properties of the induced sampling process. We isolate the property required by the outer-loop analysis in the following assumption.

For each critic, define a best parameter in the projection set by
\[
\widehat\theta_{u,k}\in\arg\min_{\|\theta\|_2\leq R}\left\|Q_{u^g}^{\pi^k}-\Phi\theta\right\|_2,
\]
and define the uniform function-approximation error as
\[
\varepsilon_{\mathrm{FA},u}:=\sup_k\left\|Q_{u^g}^{\pi^k}-\Phi\widehat\theta_{u,k}\right\|_\infty.
\]
Then we assume the following: 
\begin{assumption}[Linear TD critic rate]
\label{ass:linear-td-rate}
For the fixed feature representation and history length $l$, there exist nonnegative constants $C_{u,l}^{\mathrm{hist}}$, $C_{u,l}^{\mathrm{init}}$, $C_{u,l}^{\mathrm{stat}}$, and $c_{u,l}>0$, independent of $k$, $K$, and $\tau$, such that, for every $u\in\{r,c\}$, Algorithm~\ref{alg:tdestimation} satisfies
\[
\begin{aligned}
\delta_{u^g}(\tau,l)
&:=
\sup_{K\geq1}\sup_{0\leq k<K}
\mathbb E\!\left[
\left\|
\widehat Q_{u^g}^{\pi^k}
-
Q_{u^g}^{\pi^k}
\right\|_\infty
\right].\\
&\leq\varepsilon_{\mathrm{FA},u}+C_{u,l}^{\mathrm{hist}}\frac{(1-\kappa)^l}{1-\gamma}+C_{u,l}^{\mathrm{init}}e^{-c_{u,l}\sqrt\tau}+C_{u,l}^{\mathrm{stat}}\frac{\log\tau}{\sqrt\tau}.
\end{aligned}
\]
Moreover, the required initial phase satisfies
\[
\tau'=O_l(\log(\tau)).
\]
\end{assumption}
where the subscript $l$ in $O_l$ means it depend on constants that depend on fixed $l$.

Assumption~\ref{ass:linear-td-rate} isolates the finite-sample property required by the outer-loop analysis. The first term is the linear function-approximation error. The second accounts for the mismatch between the full-history CPOMDP dynamics generating the samples and the finite-history superstate model whose values are estimated. The last two terms are respectively the initialization and finite-sample TD errors. Because rewards and costs have the same bounded range, the same form of the TD argument can be applied separately to the two critics.

The constants may depend on the fixed feature representation, projection radius, discount factor, and the relevant mixing and coverage properties, but not on $K$ or $\tau$. The form of the bound in this assumption follows the finite-time TD analysis of \citet{anjarlekar2025scalable}: under persistent exploration and a corresponding mixing condition, their TD bound has the same structure as the one in Assumption \ref{ass:linear-td-rate}. In our setting, we additionally require the same constants to apply uniformly over all policy generated by Algorithm \ref{alg:politex} and all iteration budgets $K$. The exponential-weights update does not by itself guarantee the persistent
exploration condition, so the resulting sample-complexity bound is conditional on these uniform coverage conditions. Establishing sufficient conditions under which the policy sequence generated by Algorithm~\ref{alg:politex} satisfies the required uniform mixing and coverage properties is left for future work.

Since
\[
Q_{r^g-\lambda^k c^g}^{\pi^k}=Q_{r^g}^{\pi^k}-\lambda^kQ_{c^g}^{\pi^k},\qquad \widehat Q_{r^g-\lambda^k c^g}^{\pi^k}=\widehat Q_{r^g}^{\pi^k}-\lambda^k\widehat Q_{c^g}^{\pi^k},
\]
the triangle inequality and $\lambda^k\leq\lambda^{\max}$ give
\[
\mathbb E\!\left[\left\|\widehat Q_{r^g-\lambda^k c^g}^{\pi^k}-Q_{r^g-\lambda^k c^g}^{\pi^k}\right\|_\infty\right]\leq\delta_{r^g}(\tau,l)+\lambda^{\max}\delta_{c^g}(\tau,l).
\]
Moreover, by the definition of the estimated cost value,
\[
\begin{aligned}
\mathbb E\!\left[\left|\widehat V_{c^g}^{\pi^k}(\rho^g)-V_{c^g}^{\pi^k}(\rho^g)\right|\right]&\leq\mathbb E\!\left[\left\|\widehat Q_{c^g}^{\pi^k}-Q_{c^g}^{\pi^k}\right\|_\infty\right]\\
&\leq\delta_{c^g}(\tau,l).
\end{aligned}
\]
Consequently, the conditions of Estimation Oracle $\mathcal B$ are satisfied with
\[
\delta_{\mathrm{est}}(\tau,l):=\max\left\{\delta_{r^g}(\tau,l)+\lambda^{\max}\delta_{c^g}(\tau,l),\ \delta_{c^g}(\tau,l)\right\}.
\]
A valid upper bound is
\[
\delta_{\mathrm{est}}(\tau,l)\leq(1+\lambda^{\max})\max\left\{\delta_{r^g}(\tau,l),\delta_{c^g}(\tau,l)\right\}.
\]

\subsection{Sample Complexity Bound for Fixed Window Length}

For any fixed $l$ and admissible superstate threshold $\sd$, recall
\[
\delta(l)=\frac{2(1-\kappa)^l}{(1-\gamma)^2},\qquad \sxi=\sd-\od+\oxi-\delta(l),\qquad \lambda^{\max}=\frac{2}{(1-\gamma)\sxi}.
\]
Since the history length and feature representation are fixed, the superstate space, feature dimension, projection radius, and the relevant mixing and coverage constants do not depend on $K$ or $\tau$. For a general linear representation, $R$ is a fixed design parameter defining the critic class.

The terms in Assumption~\ref{ass:linear-td-rate} that do not vanish as $\tau$ increases induce the estimation-error floor
\[
\begin{aligned}
\delta_{\mathrm{est}}^{\mathrm{floor}}(l,\lambda^{\max}):=\max\Bigg\{&\varepsilon_{\mathrm{FA},r}+\lambda^{\max}\varepsilon_{\mathrm{FA},c}+\left(C_{r,l}^{\mathrm{hist}}+\lambda^{\max}C_{c,l}^{\mathrm{hist}}\right)\frac{(1-\kappa)^l}{1-\gamma} , \varepsilon_{\mathrm{FA},c}+C_{c,l}^{\mathrm{hist}}\frac{(1-\kappa)^l}{1-\gamma}\Bigg\}.
\end{aligned}
\]
It follows from Assumption~\ref{ass:linear-td-rate} that, for some fixed $c_l>0$,
\[
\delta_{\mathrm{est}}(\tau,l)\leq\delta_{\mathrm{est}}^{\mathrm{floor}}(l,\lambda^{\max})+O_l\!\left(e^{-c_l\sqrt\tau}+\frac{\log\tau}{\sqrt\tau}\right).
\]

Combining this inequality with Theorem \ref{thm:theorem}, the average optimality gap satisfies
\[
\begin{aligned}
\mathrm{OpGap}_K\leq\underbrace{2\delta(l)+\frac{[\od-\sd+\delta(l)]_+}{(1-\gamma)\sxi}+\frac{2(\sxi+1)}{(1-\gamma)\sxi}\delta_{\mathrm{est}}^{\mathrm{floor}}(l,\lambda^{\max})}_{=:F_r(l,\sd) }+O_l\!\left(e^{-c_l\sqrt\tau}+\frac{\log\tau}{\sqrt\tau}\right)+O_l(K^{-1/2}).
\end{aligned}
\]
Similarly, the average constraint violation satisfies
\[
\begin{aligned}
\mathrm{CViol}_K\leq\underbrace{[\sd-\od+\delta(l)]_++2(\sxi+1)\delta_{\mathrm{est}}^{\mathrm{floor}}(l,\lambda^{\max})}_{=:F_c(l,\sd)}+ O_l\!\left(e^{-c_l\sqrt\tau}+\frac{\log(\tau)}{\sqrt\tau}\right)+ O_l(K^{-1/2}).
\end{aligned}
\]
The quantities $F_r(l,\sd)$ and $F_c(l,\sd)$ are respectively the fixed-window reward and constraint floors. One can consider, for example, the case in which $\sd=\od-\delta(l)$ or $\sd=\od$. They collect the CPOMDP-to-superstate approximation error, the linear function-approximation error, and the sampling-mismatch error of the TD oracle. The remaining terms are the finite-sample TD error and the finite-$K$ optimization error. The notation $ O_l$ indicates that the hidden constants may depend on fixed quantities such as the feature representation, $R$, the mixing and coverage constants, $\sxi$, $\lambda^{\max}$, and $\gamma$, but not on $K$ or $\tau$.

Define
\[
F_{l}:=\max\{F_{r}(l, \sd),F_{c}(l, \sd)\}.
\]
Because $l$ is fixed, $F_{l}$ does not decrease when more samples are collected. Therefore, the preceding bounds certify the target accuracy only when $\varepsilon>F_{l}$. For $\varepsilon\leq F_{l}$, the guarantee does not hold.

For any $\varepsilon>F_{l}$, define the remaining error budget
\[
\varepsilon_{\mathrm{rem}}:=\varepsilon-F_{l}>0.
\]
It is sufficient to choose $K$ and $\tau$ so that the finite-sample TD and optimization errors are each of order $\varepsilon_{\mathrm{rem}}$. For $0<\varepsilon_{\mathrm{rem}}<1$, sufficient choices are
\[
\tau= O_l\!\left(\varepsilon_{\mathrm{rem}}^{-2}\log^2\frac{e}{\varepsilon_{\mathrm{rem}}}\right),\qquad K= O_l(\varepsilon_{\mathrm{rem}}^{-2}).
\]
Each of the $K$ calls to the oracle uses $\tau'$ initial interactions followed by $\tau$ TD interactions. Since the same interactions update both critics, the total number of observation-action pairs is
\[
N_{\varepsilon}=K\bigl(\tau+\tau'\bigr).
\]
Using $\tau'=O_l(\log(\tau))=o(\tau)$ gives
\[
N_{\varepsilon}=O_l\!\left(\varepsilon_{\mathrm{rem}}^{-4}\log^2\frac{1}{\varepsilon_{\mathrm{rem}}}\right).
\]

\section{Experiments}
\label{app:experiments}

\subsection{Experimental Setup}
\paragraph{Hardware Specification}
Experiments are conducted on Apple Mac Studio equipped with Apple M4 Max chip (16-core CPU, 50-core GPU), 64 GB unified memory. 

\paragraph{Environment Details}
The state space consists of the grid cell coordinates. The action space is $\mathcal{A}=\{\texttt{up},\texttt{right},\texttt{down},\texttt{left}\}$. The observation space is $\mathcal{O}=\{\texttt{safe},\texttt{edge},\texttt{near-hazard},\texttt{unsafe},\texttt{goal}\}$.
The agent starts from the middle cell of the left boundary. The absorbing goal is located at the opposite of the grid. The agent receives reward $+1$ whenever the current state is the absorbing goal state, $0$ otherwise. The cost is $1$ when the current state is unsafe, $0$ otherwise. The intended action is executed with probability $0.7$, while with probability $0.3$ one of the four action is sampled uniformly. Boundary-crossing actions leave the agent in place. The nominal observation is emitted with probability $0.9$, while with probability $0.1$ one of the five observations is sampled uniformly. We use the discount factor $\gamma=0.9$.

\paragraph{Algorithm Configuration}
We consider history lengths $l \in \{0, 1, 2, 3, 4\}$ and iteration budgets $K\in\{100,500,1000,5000,10000\}$. The original and superstate constraint threshold are set to $d^o=d^g=0.692$: this specific value is chosen so that, for every tested history length $l$, the corresponding superstate \ac{CMDP} admits a strictly feasible policy, satisfying Slater condition. For every $l$, we verify strict feasibility of the corresponding superstate CMDP by computing $\hat{\xi}^g_l = \sd - \min_{\pi\in\sPi} \sVcpi (\srho)>0$.  Since the choice of $\sd=\od$ is motivated by the fact the the filter stability coefficient $\kappa$, which is needed for the choice of $\sd=\od-\delta(l)$, is not assumed to be known in the experiments. We still perform some experiments on the constraint tightening.  
  
We initialize the policy uniformly over actions and set $\lambda^0 = 0$. The dual variable is projected onto $\left[ 0, \frac{2}{(1-\gamma)\hat{\xi}^g_l} \right]$.
All results in the main experiment use exact policy evaluation in the superstate \ac{CMDP}. 

Regarding the learning rates, in most experiments, we use $\eta_1 = \sqrt{\frac{2\log|\mathcal A|}{K}}$ and $\eta_2 = \frac{1} {\left((1-\gamma)^{-1}-\sd\right)\sqrt K}$. These have the same $K^{-1/2}$ dependence as those in Theorem~\ref{thm:theorem}, but use different multiplicative constants, alleviating the need of knowing critical parameters such as $\sxi$ or $\oxi, \kappa$. Later in the appendix we will also analyze the case with exact constants. Our objective in these experiments is to study the practical behavior of the algorithm.

The primal rate is an exponential weights calibration based on the size of the action space. For the dual update, since $\sVcpi (\srho) - \sd \le 1/(1-\gamma) - \sd$, our choice guarantees that the largest possible positive dual-update increment is scaled as 
\[
\eta_2 \left( \frac{1}{1-\gamma}-\sd \right) = \frac{1}{\sqrt K}.
\]
This normalization makes the scale of the dual updates comparable across iteration budgets.
Both rules are fixed before running and are used for all the history lengths $l$. For completeness, we also conduct additional experiments using exact learning rates from Theorem~\ref{thm:theorem}. The results are reported below. 

For every pair $(l,K)$, we evaluate the average reward value and average constraint violation of the $K$ policy iterates in the original \ac{CPOMDP}. At the beginning of each episode, one of the $K$ iterates is sampled uniformly and used throughout that episode.
We use four independent random seeds, with $5{,}000$ episodes per seed, for a total of $20{,}000$ episodes per $(l,K)$ configuration. Since the policy-evaluation oracle is exact, the policy sequence is deterministic; the seeds affect only the Monte Carlo evaluation in the original CPOMDP.
Each rollout has horizon $H = \left\lceil \frac{\log(10^{-7})}{\log\gamma} \right\rceil = 153.$
The omitted discounted tail is therefore controlled by $\gamma^H\leq10^{-7}$.

\paragraph{Construction of the SARSOP Reference Value and Visualization}
\label{app:sarsop_baseline}
The purpose of this baseline is to obtain an approximation of the optimal constrained value $\oVrpistar(\orho)$ in the original \ac{CPOMDP}. We use SARSOP~\citep{kurniawati2008sarsop}, a point-based \acs{POMDP} solver that approximates the belief space reachable under optimal policies. Unlike our algorithm, this baseline does not truncate histories and the complete belief over all $49$ latent Gridworld states is updated using the exact transition and observation models. 

Since SARSOP does not directly support constrained \acsp{POMDP}, we consider the Lagrangian. For each $\lambda \ge 0$, we solve the unconstrained problem with reward $r-\lambda c$. We first solve it for $\lambda =0$ and then evaluate $\lambda=1,2,4,\ldots$ until finding a policy satisfying
\[
\bar{\oVcpi} (\orho)+2\,\mathrm{SE}\!\left(\bar\oVcpi(\orho)\right)\le \od.
\]
where $\bar\oVcpi(\orho)$ denote the empirical mean cost value computed from the evaluation rollouts, and $\mathrm{SE}(\bar\oVcpi(\orho))$ denote its standard error.
The resulting multiplier interval is refined by at most $12$ bisection steps, with relative tolerance $10^{-3}$. Each scalar problem is solved with target precision $10^{-2}$ and a timeout of $1200$ seconds. 
We then construct an episode-level mixture of one feasible and one infeasible candidate, choosing the mixture with the largest estimated reward while satisfying the constraint with a $1.96$-standard-error margin. A final evaluation over $50{,}000$ independent rollouts gives
\[
\bar V_{r^o}^{\pi_{\mathrm{mix}}}(\rho)=3.426219\pm0.006847,\qquad\bar V_{c^o}^{\pi_{\mathrm{mix}}}(\rho)=0.684417\pm0.001827,
\]
where uncertainties denote standard errors. Moreover,
\[
    0.684417+1.96(0.001827)=0.687998<\od,
\]
confirming feasibility at the adopted statistical margin. We use the resulting reward as a feasible SARSOP-based approximation of $\oVrpistar(\orho)$.

SARSOP returns a policy over belief states. We estimate its behavior at each latent cell from rollout frequencies. The arrow shows the frequency of each action. 

\subsection{Other Results}
\paragraph{Performance in the Superstate \ac{CMDP}}
Fig.~\ref{fig:superstate_results} reports the signed average optimality gap and average constraint violation in the superstate \ac{CMDP} under exact policy evaluation. The constraint violation decreases with $K$ for all tested history lengths. For some $l$, the positive average optimality gap approximately follows the displayed $K^{-1/2}$ reference slope. For other $l$ values, it becomes negative. A negative signed gap can occur because the finite-$K$ mixture remains infeasible, attaining a reward larger than the constrained optimum. Average optimality gap and average constraint violation curves must consequently be interpreted jointly. Overall, these curves are qualitatively compatible with the $O(K^{-1/2})$ optimization terms. 
\begin{figure}
    \centering
        \includegraphics[width=0.5\linewidth]{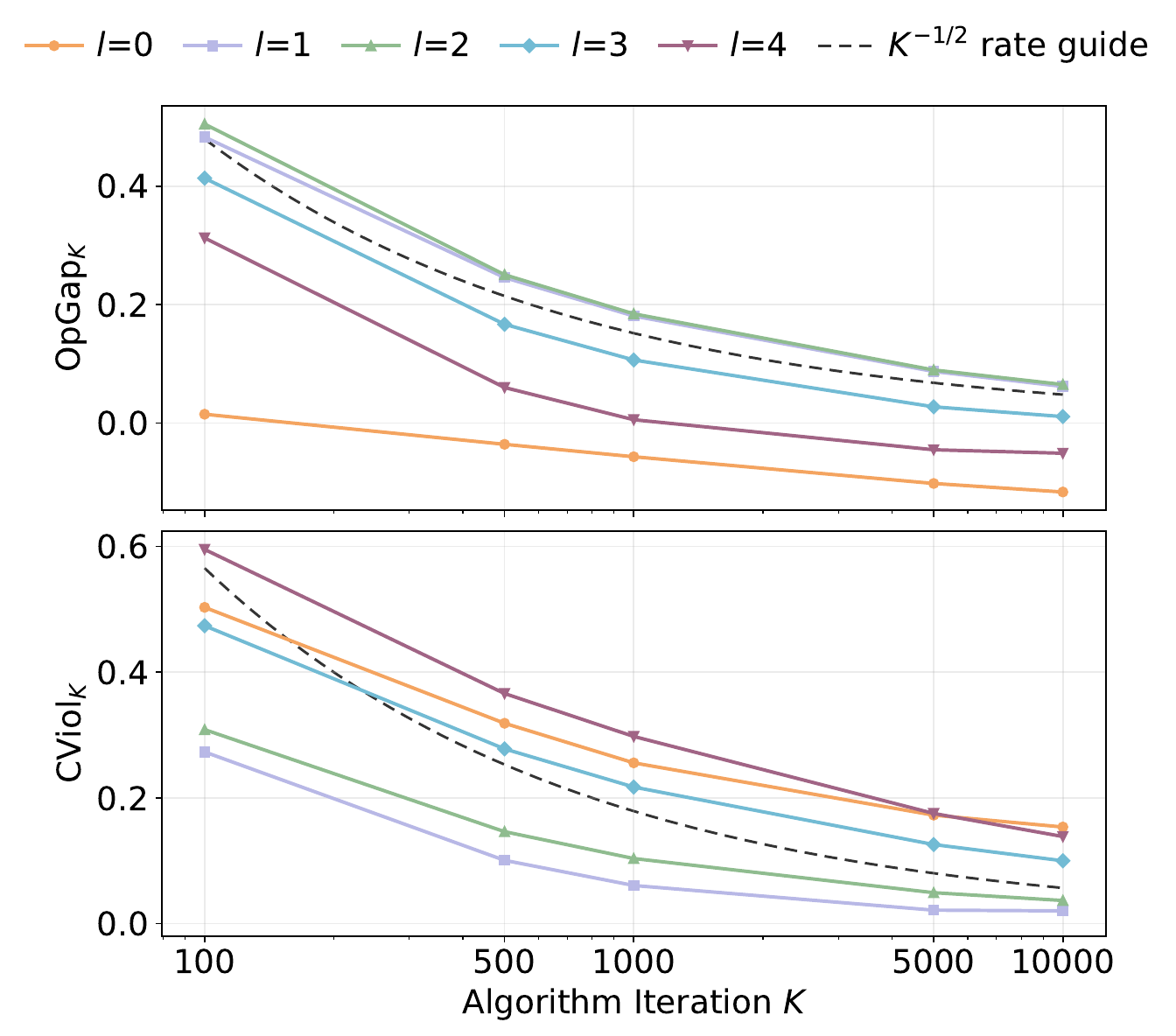}
    \caption{Performance of Algorithm~\ref{alg:politex} in the finite-history superstate model for different values of $K$ and $l$.}
    \label{fig:superstate_results}
\end{figure}

\paragraph{Effect of Implementing an Estimation Oracle}
We use persistent tabular TD as an empirical implementation of the estimation oracle. The experiment evaluates the practical effect of approximate policy evaluation, it is not intended as an empirical verification of an end-to-end TD sample-complexity bound. At each primal-dual iteration, the TD(0) oracle collects 64 trajectories of horizon 100 under the current policy. An additional 2000 warm-up trajectories are collected during the first iteration. The action-value tables persist across the primal-dual iterations, and, therefore, are not re-initialized. For the $n$th visit to a superstate-action pair we use as stepsize:
\[
    \alpha_n = \frac{0.1}{(1+0.01n)^{1/2}}.
\]

Results for $l\in \{0, 1, 2, 3, 4\}$ and $K\in \{100, 500, 1000, 5000, 10000\}$ are reported in Fig.~\ref{fig:oracleeffect}. 
Compared with exact policy evaluation, the TD oracle introduces in general an additional bias in both reward and constraint performance. This is qualitatively consistent with the estimation-error terms
$\propto\delta_{\mathrm{est}}$ appearing in the bounds. Since the TD estimates are retained across iterations, their actual errors vary with $K$. This illustrates the viability and the effect of approximate policy evaluation.

\begin{figure}
    \centering
    \includegraphics[width=\linewidth]{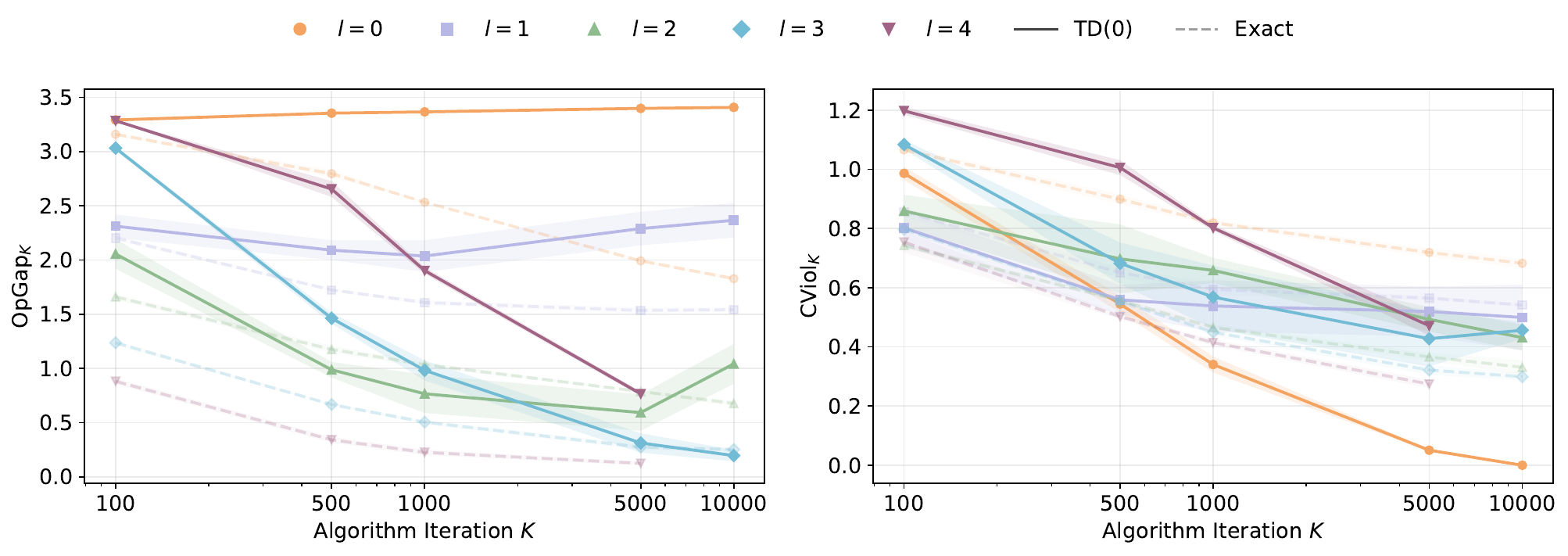}
    \caption{Performance of using exact oracle vs TD(0) learning oracle in Algorithm~\ref{alg:politex}. Average optimality gap (left) and average constraint violation (right). The iterates are evaluated in the \ac{CPOMDP}.}
    \label{fig:oracleeffect}
\end{figure}

\paragraph{Effect of Superstate Constraint Tightening}
Fig.~\ref{fig:tightening} illustrates the effect of tightening the superstate constraint for $l=3$, using smaller $\sd$. Let
\[
\Delta_{\mathrm{slack}} := \od- \min_{\pi\in\sPi} \sVcpi(\srho) \simeq 0.4753
\]
denote the available strict-feasibility slack. We set
\[
\sd_\alpha=\od-\alpha\Delta_\mathrm{slack}, \qquad \alpha\in\{0,0.25,0.5,0.75\},
\]
obtaining
\[
\sd_\alpha \in\{0.692,\,0.573,\,0.454,\,0.336\}.
\]
Hence, $\sd_0=0.692$ corresponds to the untightened problem, while smaller values impose increasingly conservative superstate constraints. The results show the expected reward-safety tradeoff: stronger tightening reduces the cost incurred in the original CPOMDP, at the price of a lower reward.
\begin{figure}
    \centering
    \begin{subfigure}[t]{0.49\textwidth}
        \centering
        \includegraphics[width=\linewidth]
            {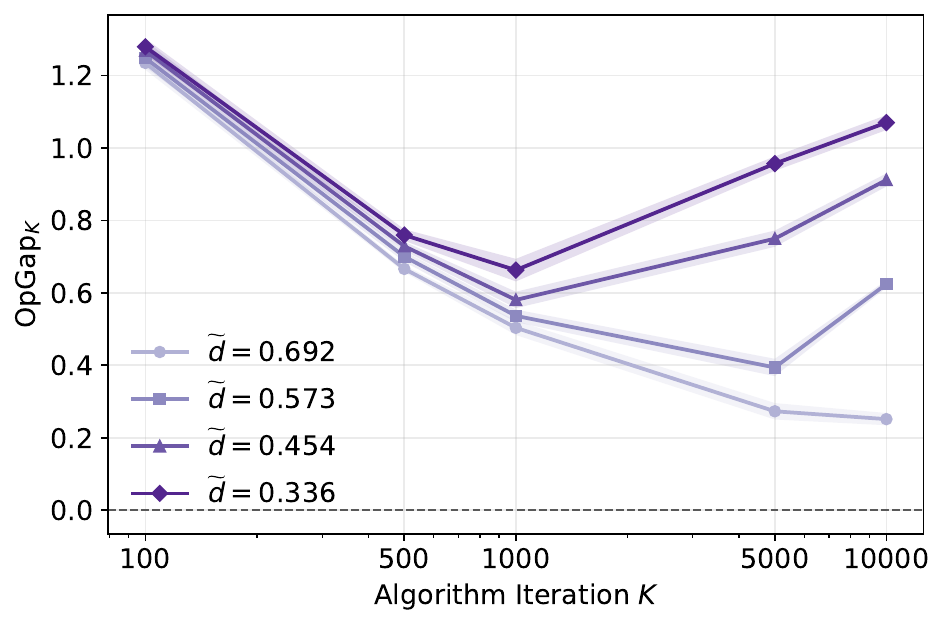}
    \end{subfigure}
    \hfill
    \begin{subfigure}[t]{0.49\textwidth}
        \centering
        \includegraphics[width=\linewidth]
            {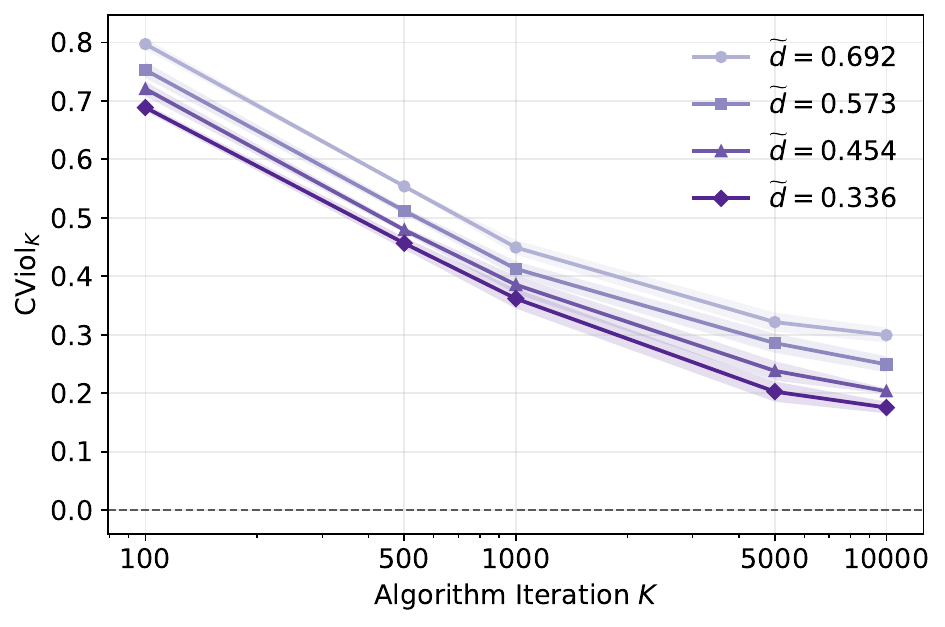}
    \end{subfigure}

    \caption{Effect of superstate constraint tightening. Average optimality gap (left) and average constraint violation (right). The iterates are evaluated in the \ac{CPOMDP}. History length $l=3$. }
    \label{fig:tightening}
\end{figure}

\paragraph{Fully Observable and Unconstrained Partially Observable Baselines}
\,\\
\textbf{1. Model-Based Planning References.}
Table~\ref{tab:reference_policies} compares three reference policies. The approximate constrained \ac{CPOMDP} optimum achieves an expected discounted cumulative reward of $3.426$ and an expected cumulative cost of $0.684$. Removing the constraint, but keeping partial observability, increases the SARSOP reward to $3.898$, but also raises the cost to $1.302$, above the threshold $\od \simeq 0.692$.
Under full observability, the exact constrained optimum achieves a reward of $4.162$ with a cost of $0.692$, showing the price of partial observability.

\begin{table}[h!]
\centering
\setlength{\tabcolsep}{3pt}
\begin{tabular}{lcc}
\toprule
Policy & Reward & Cost \\
\midrule
Approx.\ constrained \ac{CPOMDP}
    & $3.426 \pm 0.013$
    & $0.684 \pm 0.004$ \\
Approx.\ unconstrained \ac{POMDP}
    & $3.898 \pm 0.020$
    & $1.302 \pm 0.013$ \\
Exact constrained \ac{CMDP}
    & $4.162$
    & $0.692$ \\
\bottomrule
\end{tabular}
\caption{Model-based planning references. SARSOP-based values are
reported with 95\% confidence intervals. The constraint threshold is
$\od \simeq 0.692$.}
\label{tab:reference_policies}
\end{table}

Fig.~\ref{fig:optimalvssarsop} illustrates the difference between the constrained partially observable policy and constrained fully observable policies. The former appears more conservative, tending to move toward the boundary of the grid rather than through the central corridor when the agent is in some unsafe cells. This behavior is consistent with the state uncertainty induced by noisy and aliased observations. 
\begin{figure}[h]
    \centering
    \begin{subfigure}[t]{0.49\textwidth}
        \centering
        \includegraphics[width=0.9\linewidth]
            {images/sarsop_policy_map.png}
    \end{subfigure}
    \hfill
    \begin{subfigure}[t]{0.49\textwidth}
        \centering
        \includegraphics[width=0.9\linewidth]
            {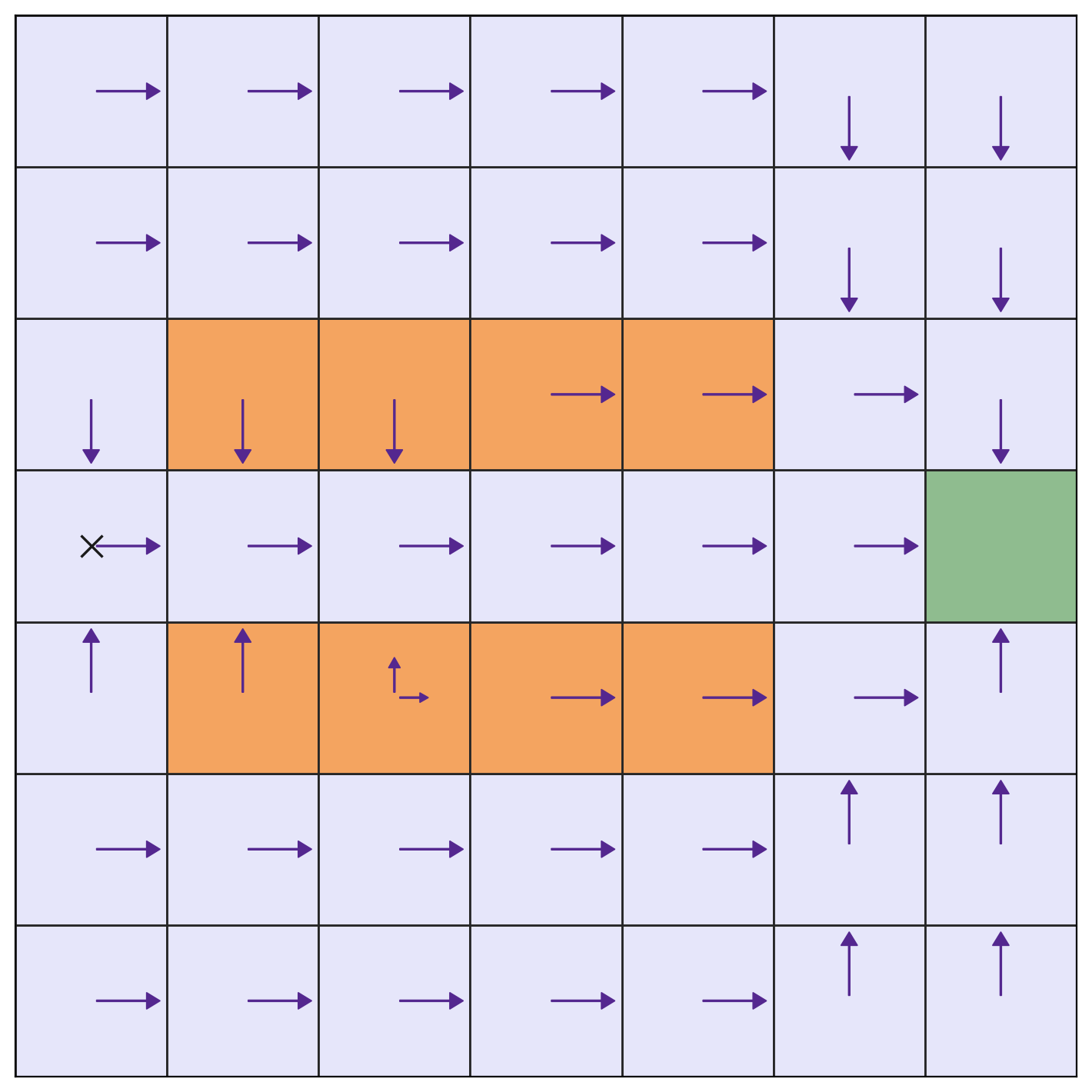}
    \end{subfigure}

    \caption{Approximate constrained \ac{CPOMDP} policy obtained using SARSOP (left) and exact fully observable policy (right). Their respective expected cumulative reward values are $3.426$ and $4.162$, while their cost values are $0.684$ and $0.692$. Arrow lengths are proportional to action probabilities.}
    \label{fig:optimalvssarsop}
\end{figure}

\textbf{2. Unconstrained and Fully Observable Learning Baselines.}
Table~\ref{tab:cpomdp_constrained_unconstrained} compares the constrained and unconstrained variants of Algorithm~\ref{alg:politex} after $K=10^4$ iterations. The constrained variant yields lower constraint violation for every history length. The unconstrained values are not uniformly higher because the table
reports finite-iteration average policies learned in the superstate
model and evaluated in the original \ac{CPOMDP}, rather than exact
optimal policies. Hence, the usual ordering between constrained and
unconstrained optimal values need not hold.

\begin{table}[h!]
\centering
\setlength{\tabcolsep}{3.5pt}
\begin{tabular}{c cc cc}
\toprule
& \multicolumn{2}{c}{\textbf{Average reward}}
& \multicolumn{2}{c}{\textbf{Constraint violation}} \\
$l$ & Constr. & Unconstr. & Constr. & Unconstr. \\
\midrule
0 & 1.599 & 0.685 & 0.682 & 1.158 \\
1 & 1.885 & 1.889 & 0.541 & 0.695 \\
2 & 2.750 & 2.901 & 0.331 & 0.840 \\
3 & 3.175 & 2.958 & 0.299 & 1.504 \\
4 & 3.329 & 3.534 & 0.237 & 0.922 \\
\bottomrule
\end{tabular}
\caption{Constrained and unconstrained policies evaluated in the
original \ac{CPOMDP} after $K=10^4$ iterations. Reported values are
averages over four random seeds. Constraint violations are computed
relative to $\od \simeq 0.692$.}
\label{tab:cpomdp_constrained_unconstrained}
\end{table}

Finally, Fig.~\ref{fig:fullyobs} compares Algorithm~\ref{alg:politex} performances under partial and full observability, showing that the performance gap decreases as $l$ increases.
\begin{figure}[h!]
    \centering
    \includegraphics[width=0.5\linewidth]
        {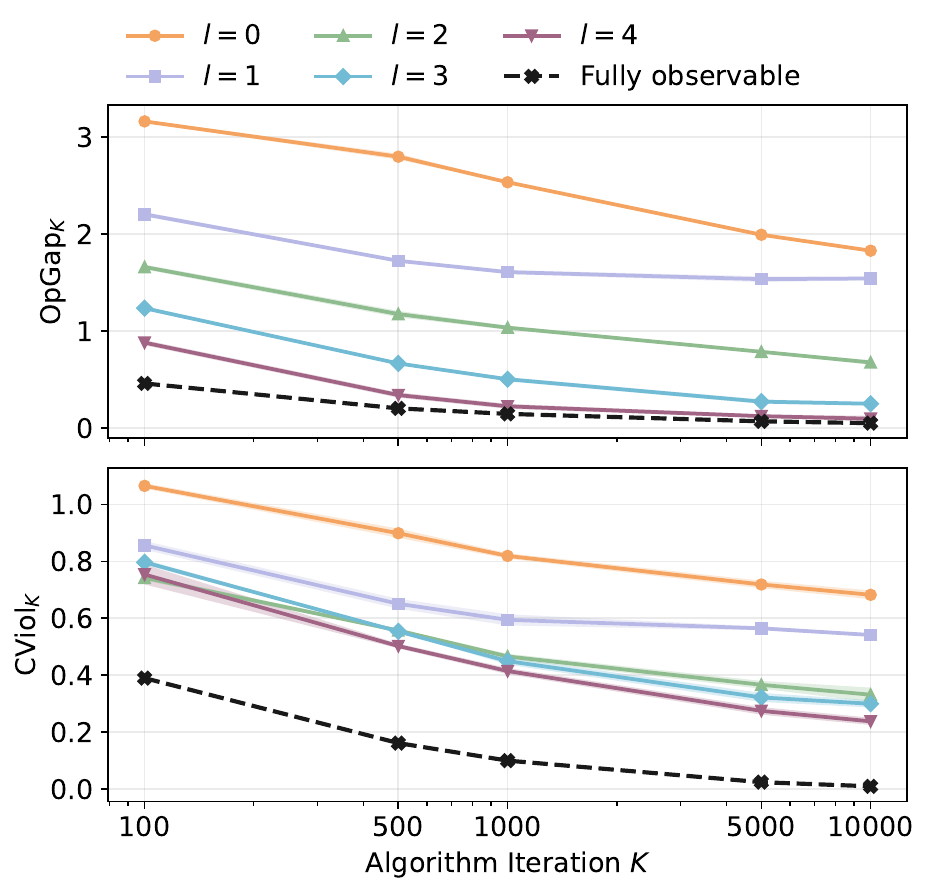}
    \caption{Performance of Algorithm~\ref{alg:politex} under partial
    observability for different values of $K$ and $l$, compared with
    its fully observable performance (black line).}
    \label{fig:fullyobs}
\end{figure}

\paragraph{Matching the Theoretical stepsizes}
Figure~\ref{fig:step_sizes} reports the performance of Algorithm~\ref{alg:politex} using the exact step-size constants from Theorem~\ref{thm:theorem},
\[
\eta_1 =
\frac{(1-\gamma)^2\xi^g}{(1-\gamma)\xi^g+2}
\sqrt{\frac{2\log|\mathcal A|}{K}},
\qquad
\eta_2 =
\frac{2}{\xi^g(1+(1-\gamma)d^g)\sqrt K}.
\]
Note that here we report the average reward value rather than the average optimality gap for better visualization.
Compared with the practical values used in the main experiments, the constraint violation decreases further, whereas the reward performances improve much more slowly. This is explained by the conservative theorem constants. For example, at \(K=1000\), the practical calibration gives
\(\eta_1=0.0526\) and \(\eta_2=0.0034\). With the theorem constants,
\(\eta_1\) ranges from \(1.74\cdot10^{-5}\) to
\(1.38\cdot10^{-4}\), while \(\eta_2\) ranges from \(0.110\) to
\(0.894\), depending on \(l\) through \(\xi_l^g\). Thus, the theorem calibration makes the primal updates substantially smaller and the dual updates larger. This explains the stronger reduction in constraint violation and the considerably slower reward improvement and the differences between curves associated with different $l$.

\begin{figure}
    \centering
    \begin{subfigure}[t]{0.49\textwidth}
        \centering
        \includegraphics[width=\linewidth]{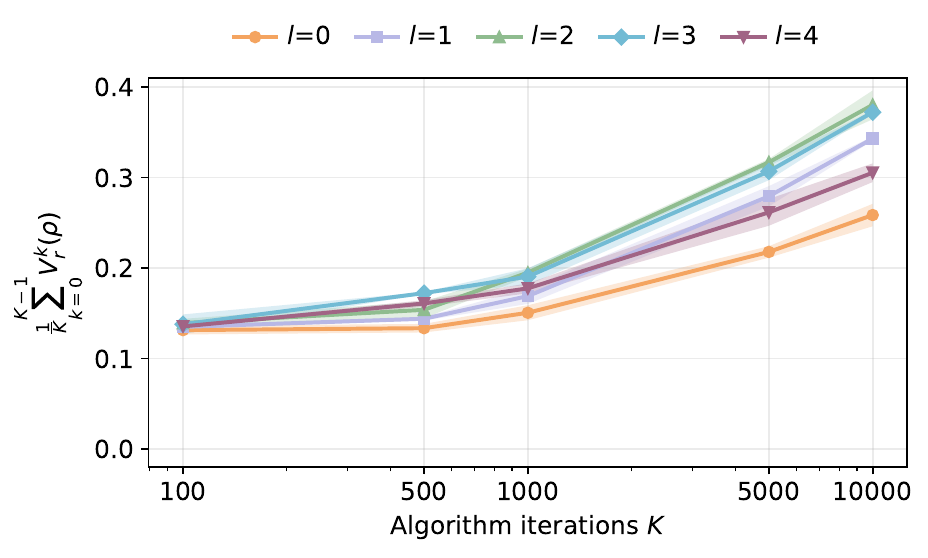}
    \end{subfigure}
    \hfill
    \begin{subfigure}[t]{0.49\textwidth}
        \centering
        \includegraphics[width=\linewidth]{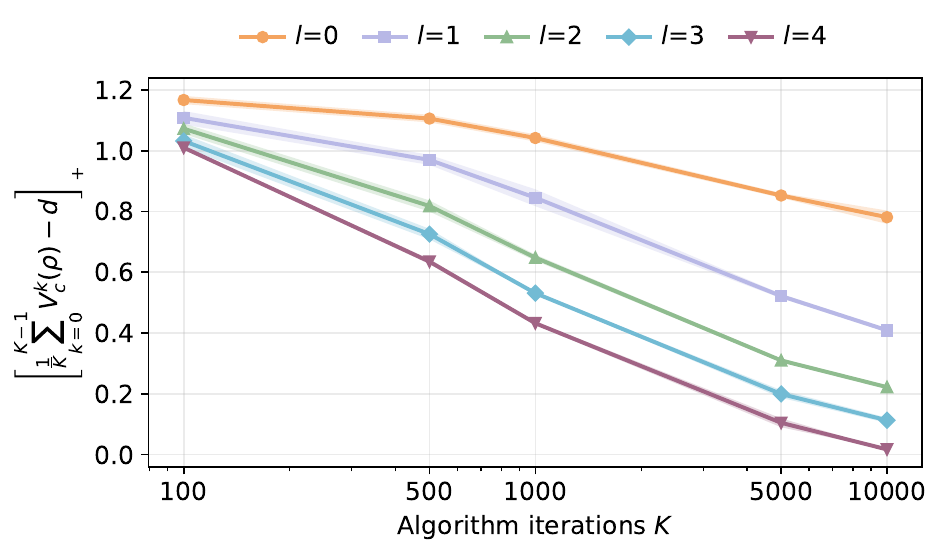}
    \end{subfigure}

    \caption{Performance of Algorithm~\ref{alg:politex} in the \ac{CPOMDP} for different values of $K$ and $l$ using the stepsizes required by Theorem~\ref{thm:theorem}. Left:
    $\frac{1}{K}\sum_{k=0}^{K-1} \oVrpik(\orho)$.
    Right:
    $\left[\frac{1}{K}\sum_{k=0}^{K-1}\oVcpik(\orho)-\od\right]_+$.}
    \label{fig:step_sizes}
\end{figure}

%% file: main.bbl
\begin{thebibliography}{}

\bibitem[Abbasi-Yadkori et~al., 2019]{abbasi2019politex}
Abbasi-Yadkori, Y., Bartlett, P., Bhatia, K., Lazic, N., Szepesvari, C., and Weisz, G. (2019).
\newblock Politex: Regret bounds for policy iteration using expert prediction.
\newblock In {\em International conference on machine learning}, pages 3692--3702. PMLR.

\bibitem[Altman, 2021]{altman2021constrained}
Altman, E. (2021).
\newblock {\em Constrained Markov decision processes}.
\newblock Routledge.

\bibitem[Anjarlekar et~al., 2025]{anjarlekar2025scalable}
Anjarlekar, A., Etesami, R., and Srikant, R. (2025).
\newblock Scalable policy-based rl algorithms for {POMDP}s.
\newblock {\em arXiv preprint arXiv:2510.06540}.

\bibitem[Carr et~al., 2023]{carr2023safe}
Carr, S., Jansen, N., Junges, S., and Topcu, U. (2023).
\newblock Safe reinforcement learning via shielding under partial observability.
\newblock In {\em Proceedings of the AAAI conference on artificial intelligence}, volume~37, pages 14748--14756.

\bibitem[Cayci et~al., 2024]{cayci2024finite}
Cayci, S., He, N., and Srikant, R. (2024).
\newblock Finite-time analysis of natural actor-critic for {POMDP}s.
\newblock {\em SIAM Journal on Mathematics of Data Science}, 6(4):869--896.

\bibitem[Cesa-Bianchi and Lugosi, 2006]{cesa2006prediction}
Cesa-Bianchi, N. and Lugosi, G. (2006).
\newblock {\em Prediction, learning, and games}.
\newblock Cambridge university press.

\bibitem[Ding et~al., 2023]{ding2023last}
Ding, D., Wei, C.-Y., Zhang, K., and Ribeiro, A. (2023).
\newblock Last-iterate convergent policy gradient primal-dual methods for constrained {MDP}s.
\newblock {\em Advances in Neural Information Processing Systems}, 36:66138--66200.

\bibitem[Ding et~al., 2025]{ding2025convergence}
Ding, D., Zhang, K., Duan, J., Basar, T., and Jovanovic, M.~R. (2025).
\newblock Convergence and sample complexity of natural policy gradient primal-dual methods for constrained {MDP}s.
\newblock {\em Journal of Machine Learning Research}, 26(256):1--76.

\bibitem[Ding et~al., 2018]{ding2018adaptive}
Ding, L., Li, S., Gao, H., Chen, C., and Deng, Z. (2018).
\newblock Adaptive partial reinforcement learning neural network-based tracking control for wheeled mobile robotic systems.
\newblock {\em IEEE Transactions on Systems, Man, and Cybernetics: Systems}, 50(7):2512--2523.

\bibitem[Efroni et~al., 2020]{efroni2020exploration}
Efroni, Y., Mannor, S., and Pirotta, M. (2020).
\newblock Exploration-exploitation in constrained {MDP}s.
\newblock {\em arXiv preprint arXiv:2003.02189}.

\bibitem[Hern{\'a}ndez-Lerma and Gonz{\'a}lez-Hern{\'a}ndez, 2000]{hernandez2000constrained}
Hern{\'a}ndez-Lerma, O. and Gonz{\'a}lez-Hern{\'a}ndez, J. (2000).
\newblock Constrained markov control processes in borel spaces: the discounted case.
\newblock {\em Mathematical Methods of Operations Research}, 52(2):271--285.

\bibitem[Hogewind et~al., 2022]{hogewind2022safe}
Hogewind, Y., Simao, T.~D., Kachman, T., and Jansen, N. (2022).
\newblock Safe reinforcement learning from pixels using a stochastic latent representation.
\newblock {\em arXiv preprint arXiv:2210.01801}.

\bibitem[Huang et~al., 2025]{huang2025pigdreamer}
Huang, D., Wang, J., Li, Y., Xia, C., Zhang, T., and Zhang, K. (2025).
\newblock {PIGD}reamer: Privileged information guided world models for safe partially observable reinforcement learning.
\newblock {\em arXiv preprint arXiv:2508.02159}.

\bibitem[Jordan and Kamgarpour, 2026]{jordan2026model}
Jordan, P. and Kamgarpour, M. (2026).
\newblock Model-based learning of near-optimal finite-window policies in {POMDP}s.
\newblock {\em arXiv preprint arXiv:2604.01024}.

\bibitem[Kaelbling et~al., 1998]{kaelbling1998planning}
Kaelbling, L.~P., Littman, M.~L., and Cassandra, A.~R. (1998).
\newblock Planning and acting in partially observable stochastic domains.
\newblock {\em Artificial intelligence}, 101(1-2):99--134.

\bibitem[Kakade and Langford, 2002]{kakade2002approximately}
Kakade, S. and Langford, J. (2002).
\newblock Approximately optimal approximate reinforcement learning.
\newblock In {\em Proceedings of the nineteenth international conference on machine learning}, pages 267--274.

\bibitem[Kalagarla et~al., 2022]{kalagarla2022optimal}
Kalagarla, K.~C., Dhruva, K., Shen, D., Jain, R., Nayyar, A., and Nuzzo, P. (2022).
\newblock Optimal control of partially observable markov decision processes with finite linear temporal logic constraints.
\newblock In {\em Uncertainty in Artificial Intelligence}, pages 949--958. PMLR.

\bibitem[Kara and Yuksel, 2022]{kara2022near}
Kara, A. and Yuksel, S. (2022).
\newblock Near optimality of finite memory feedback policies in partially observed markov decision processes.
\newblock {\em Journal of Machine Learning Research}, 23(11):1--46.

\bibitem[Kara and Y{\"u}ksel, 2023]{kara2023convergence}
Kara, A.~D. and Y{\"u}ksel, S. (2023).
\newblock Convergence of finite memory q learning for {POMDP}s and near optimality of learned policies under filter stability.
\newblock {\em Mathematics of Operations Research}, 48(4):2066--2093.

\bibitem[Kearns and Singh, 2002]{kearns2002near}
Kearns, M. and Singh, S. (2002).
\newblock Near-optimal reinforcement learning in polynomial time.
\newblock {\em Machine learning}, 49(2):209--232.

\bibitem[Khonji et~al., 2019]{khonji2019approximability}
Khonji, M., Jasour, A., and Williams, B.~C. (2019).
\newblock Approximability of constant-horizon constrained pomdp.
\newblock In {\em IJCAI}, pages 5583--5590.

\bibitem[Krishnamurthy et~al., 2016]{krishnamurthy2016pac}
Krishnamurthy, A., Agarwal, A., and Langford, J. (2016).
\newblock {PAC} reinforcement learning with rich observations.
\newblock {\em Advances in Neural Information Processing Systems}, 29.

\bibitem[Kurniawati et~al., 2008]{kurniawati2008sarsop}
Kurniawati, H., Hsu, D., Lee, W.~S., et~al. (2008).
\newblock Sarsop: Efficient point-based {POMDP} planning by approximating optimally reachable belief spaces.
\newblock In {\em Robotics: Science and systems}, volume 2008. Zurich, Switzerland.

\bibitem[Lauri et~al., 2022]{lauri2022partially}
Lauri, M., Hsu, D., and Pajarinen, J. (2022).
\newblock Partially observable markov decision processes in robotics: A survey.
\newblock {\em IEEE Transactions on Robotics}, 39(1):21--40.

\bibitem[Lee et~al., 2018]{lee2018monte}
Lee, J., Kim, G.-H., Poupart, P., and Kim, K.-E. (2018).
\newblock Monte-carlo tree search for constrained {POMDP}s.
\newblock {\em Advances in Neural Information Processing Systems}, 31.

\bibitem[McDonald and Y{\"u}ksel, 2020]{mcdonald2020exponential}
McDonald, C. and Y{\"u}ksel, S. (2020).
\newblock Exponential filter stability via dobrushin’s coefficient.

\bibitem[Ni and Kamgarpour, 2025]{pmlr-v258-ni25a}
Ni, T. and Kamgarpour, M. (2025).
\newblock A safe exploration approach to constrained markov decision processes.
\newblock In {\em Proceedings of The 28th International Conference on Artificial Intelligence and Statistics}, volume 258 of {\em Proceedings of Machine Learning Research}, pages 3592--3600. PMLR.

\bibitem[Papadimitriou and Tsitsiklis, 1987]{papadimitriou1987complexity}
Papadimitriou, C.~H. and Tsitsiklis, J.~N. (1987).
\newblock The complexity of markov decision processes.
\newblock {\em Mathematics of operations research}, 12(3):441--450.

\bibitem[Paternain et~al., 2019]{paternain2019constrained}
Paternain, S., Chamon, L., Calvo-Fullana, M., and Ribeiro, A. (2019).
\newblock Constrained reinforcement learning has zero duality gap.
\newblock {\em Advances in Neural Information Processing Systems}, 32.

\bibitem[Puterman, 2014]{puterman2014markov}
Puterman, M.~L. (2014).
\newblock {\em Markov decision processes: discrete stochastic dynamic programming}.
\newblock John Wiley \& Sons.

\bibitem[Sutton et~al., 1998]{sutton1998reinforcement}
Sutton, R.~S., Barto, A.~G., et~al. (1998).
\newblock {\em Reinforcement learning: An introduction}, volume~1.
\newblock MIT press Cambridge.

\bibitem[Tian et~al., 2024]{tian2024confident}
Tian, T., Yang, L.~F., and Szepesv{\'a}ri, C. (2024).
\newblock Confident natural policy gradient for local planning in q\_pi -realizable constrained {MDP}s.
\newblock {\em Advances in Neural Information Processing Systems}, 37:76139--76176.

\bibitem[Undurti and How, 2010]{undurti2010online}
Undurti, A. and How, J.~P. (2010).
\newblock An online algorithm for constrained {POMDP}s.
\newblock In {\em 2010 IEEE International Conference on Robotics and Automation}, pages 3966--3973. IEEE.

\bibitem[Yu et~al., 2021]{yu2021reinforcement}
Yu, C., Liu, J., Nemati, S., and Yin, G. (2021).
\newblock Reinforcement learning in healthcare: A survey.
\newblock {\em ACM Computing Surveys (CSUR)}, 55(1):1--36.

\end{thebibliography}
